\documentclass[11pt,draftcls,onecolumn,twoside]{IEEEtran} 

\usepackage{amsmath,amsthm,amssymb,enumerate}
\usepackage{graphicx}
\usepackage{subcaption}
\usepackage{algorithm}
\usepackage{algpseudocode}
\usepackage{mathtools}

\DeclarePairedDelimiter{\floor}{\lfloor}{\rfloor}

\newcommand{\hide}[1]{}

\newcommand{\Z}{\mathbb{Z}}
\newcommand{\R}{\mathbb{R}}

\newcommand{\A}{\vec{A}}
\newcommand{\av}{\vec{a}}
\newcommand{\x}{\vec{x}}
\newcommand{\X}{\vec{X}}
\newcommand{\y}{\vec{y}}
\newcommand{\w}{\vec{w}}

\newcommand{\T}{\top}
\renewcommand{\vec}[1]{\boldsymbol{#1}}
\newcommand\ev[1]{\mathbb{E} \left[#1\right]}

\newcommand\evc[2]{\mathbb{E}_{#1} \left[#2\right]}
\newcommand\Cov[1]{\mathrm{Cov} \left[#1\right]}

\def\P{{\rm P}\,}

\newtheorem{theorem}{Theorem}[section]
\newtheorem{lemma}[theorem]{Lemma}

\newtheorem{mydef}{Definition}
\newtheorem{rem}{Remark}
\newtheorem{corollary}[theorem]{Corollary}

\DeclareMathOperator{\sgn}{sgn}
\DeclareMathOperator{\supp}{supp}
\DeclareMathOperator*{\argmin}{argmin}

\def\NoNumber#1{{\def\alglinenumber##1{}\State #1}\addtocounter{ALG@line}{-1}}

\begin{document}

\title{Recursive Compressed Sensing}
\author{Nikolaos M. Freris,~\IEEEmembership{Member,~IEEE,}, Orhan \"{O}\c{c}al,~\IEEEmembership{Student member,~IEEE,}
        and~Martin Vetterli,~\IEEEmembership{Fellow,~IEEE}
\thanks{The authors are with the School of Computer and Communication Sciences, \'{E}cole Polytechnique F\'{e}d\'{e}rale de Lausanne (EPFL), CH-1015 Lausanne, Switzerland. %
\tt{\{nikolaos.freris, orhan.ocal, martin.vetterli\}@epfl.ch}.}%
\thanks{This work was submitted to IEEE Transactions on Information Theory--Dec. 2013. A preliminary version of this work was presented at the $51^{st}$ Allerton conference, 2013~\cite{allerton}.} 
}

\markboth{TECHNICAL REPORT} 
{Freris et al.: Recursive Compressed Sensing}

\maketitle

\begin{abstract}
We introduce a recursive algorithm for performing compressed sensing on streaming data. The approach consists of a) \emph{recursive encoding}, where we sample the input stream via overlapping windowing and make use of the previous measurement in obtaining the next one, and b) \emph{recursive decoding}, where the signal estimate from the previous window is utilized in order to achieve faster convergence in an iterative optimization scheme applied to decode the new one. To remove estimation bias, a two-step estimation procedure is proposed comprising support set detection and signal amplitude estimation. Estimation accuracy is enhanced by a non-linear voting method and averaging estimates over multiple windows. We analyze the computational complexity and estimation error, and 
show that the normalized error variance asymptotically goes to zero for sublinear sparsity. 
Our simulation results show speed up of an order of magnitude over traditional CS, while obtaining significantly lower reconstruction error under mild conditions on the signal magnitudes and the noise level.
\end{abstract}


\section{Introduction}\label{sec:introduction}

In signal processing, it is often the case that signals of interest can be represented sparsely by using few coefficients in an appropriately selected orthonormal basis or frame. For example, the Fourier basis is used for bandlimited signals, while wavelet bases are used for piecewise continuous signals--with applications in communications for the former, and image compression for the latter. While a small number of coefficients in the respective basis may be enough for high accuracy representation, the celebrated Nyquist/Shannon sampling theorem suggests a sampling rate that is at least twice the signal bandwidth, which, in many cases, is much higher than the sufficient number of coefficients \cite{Candes2008,vetterli2002}. 

The Compressed Sensing (CS)--also referred to as Compressive Sampling--framework was introduced for sampling signals not according to bandwidth, but rather to  their \emph{information content}, i.e.,, the number of degrees of freedom. This sampling paradigm suggests a lower sampling rate compared to the classical sampling theory for signals that have sparse representation in some fixed basis \cite{Candes2008}, 
or even non-bandlimited signals \cite{vetterli2002}, to which traditional sampling does not even apply. 

The foundations of CS have been developed in \cite{Donoho2006,Candes2006b}. Although the field has been extensively studied for nearly a decade, performing CS on streaming data still remains fairly open, and an efficient recursive algorithm is, to the best of our knowledge, not available. This is the topic of the current paper, where we study the architecture for compressively sampling an input data stream, and analyze the computational complexity and stability of signal estimation from noisy samples. 
%
The main contributions are:

\begin{enumerate}
\item We process the data stream by successively performing CS in sliding overlapping windows. Sampling overhead is minimized by recursive computations using a cyclic rotation of the same sampling matrix. A similar approach is applicable when the data are sparsely representable in the Fourier domain.
\item We perform recursive decoding of the obtained samples by using the estimate from a previously decoded window to obtain a warm-start in decoding the next window (via an iterative optimization method).
\item In our approach, a given entry of the data stream is sampled over  multiple windows. In order to enhance estimation accuracy, we propose a three-step procedure to combine estimates corresponding to a given sample obtained from different windows: 
\begin{itemize}
\item{\bf Support detection} amounts to estimating whether or not a given entry is non-zero. This is accomplished by a voting strategy over multiple overlapping windows containing the entry.  
\item{\bf Ordinary least-squares} is performed in each window on the determined support.
\item{\bf Averaging} of estimates across multiple windows yields the final estimate for each entry of the data stream.
\end{itemize} 
\item Extensive experiments showcase the merits of our approach in terms of substantial decrease in both run-time and estimation error.
\end{enumerate}

Similar in spirit to our approach are the works of Garrigues and El Ghaoui~\cite{recursive_LASSO}, Boufounos and Asif~\cite{boufounos} and Asif and Romberg~\cite{romberg}. In~\cite{recursive_LASSO}, a recursive algorithm was proposed for solving LASSO based on warm-start. In~\cite{boufounos}, the data stream is assumed sparse in the frequency domain, and \emph{Streaming Greedy Pursuit} is proposed for progressively filtering measurements in order to reconstruct the data stream. In~\cite{romberg}, the authors analyze the use of warm-start for speeding up the decoding step. Our work is different in that: a) it both minimizes sampling overhead by recursively encoding the data stream, as well as b) produces high-accuracy estimates by combining information over multiple windows at the decoding step.

The organization of the paper is as follows: Section~\ref{sec:background} describes the notation, definitions and the related literature on CS. Section~\ref{sec:recursivecs} introduces the problem formulation and describes the key components of Recursive CS (RCS): \emph{recursive sampling} and \emph{recursive estimation}. We analyze the proposed method and discuss extensions in Section~\ref{sec:discussion}. Experimental results on resilience to noise and execution time of RCS are reported in Section~\ref{sec:simulations}. 
\section{Background}
\label{sec:background}

This section introduces the notation and definitions used in the paper, and summarizes the necessary background on CS.

\subsection{Notation}

Throughout the paper, we use capital boldface letters to denote matrices (e.g., $\A$) and boldface lowercase letters to denote vectors (e.g., $\x$). We use $x_i$ to denote the $i^{th}$ entry of vector $\x$,  $\vec{a}_i$ to denote the $i^{th}$ column of matrix $\A$, and $A_{ij}$ to denote its $(i,j)$ entry . The $i^{th}$ sampling instance (e.g., $i^{th}$ window of the input stream, $i^{th}$ sampling matrix, $i^{th}$ sample) is denoted by superscript (e.g., $\x^{(i)}$, $\A^{(i)}$, $\y^{(i)}$). The cardinality of a set $\mathcal{S}$ is denoted by $\vert \mathcal{S} \vert$,  and we use $\{ x_i\}$ as shorthand notation for the infinite sequence $\{ x_i\}_{i = 0,1,\dots}$. Last, we use $\evc{x}{\,\cdot\,}$ to denote the conditional expectation $\evc{x}{\, \cdot \,} = \ev{\,\cdot\, \vert x}$.

\subsection{Definitions and Properties}
\label{sec:definitions}

In the following, we summarize the key definitions related to compressed sensing.

\begin{mydef}[$\kappa$-sparsity]
For a vector $\vec{x} \in \R^n$ we define the support $\supp (\vec{x}) := \{i: x_i \neq 0\}$. The $\ell_0$ pseudonorm is $\Vert \x \Vert_0 := \vert \mbox{supp}(\x)  \vert$. We say that a vector $\vec{x}$ is $\kappa$\emph{-sparse} if and only if $\Vert \x \Vert_0 \leq \kappa$. 
\end{mydef}

\begin{mydef}[Mutual Coherence]
For a matrix $\vec{A} \in \R^{m \times n}$, the mutual coherence is defined as the largest normalized inner product between any two different columns of $\A$ \cite{Bruckstein2009}:
\begin{align}
\mu(\A) \coloneqq \max_{\substack{0\leq i,j\leq n-1 \\ i \neq j}} \frac{\vert \vec{a}^{\T}_i \vec{a}_j \vert}{\Vert \av_i \Vert_2 \cdot \Vert \av_j \Vert_2}.\\
\end{align}

\end{mydef}

\begin{mydef}[Restricted Isometry Property]
Let $\vec{A} \in \R^{m \times n}$. For given $0 < \kappa < n$, the matrix $\vec{A}$ is said to satisfy the Restricted Isometry Property (RIP) if there exists $\delta_\kappa \in [0,1]$ such that:
\begin{align}
(1-\delta_\kappa)\Vert \vec{x}\Vert_2^2 \leq \Vert  \vec{Ax} \Vert_2^2 \leq (1+\delta_\kappa) \Vert \vec{x} \Vert_2^2
\end{align}
holds for all $\vec{x} \in \R^n$ $\kappa$-sparse vectors, for a constant  $\delta_\kappa \ge 0$ sufficiently small~\cite{Candes2008}. 
\end{mydef}
The value $\delta_\kappa$ is called the \emph{restricted isometry constant} of $\A$ for $\kappa$-sparse vectors. Evidently, an equivalent description of RIP is that every subset of $\kappa$ columns of $\vec{A}$ approximately behaves like an orthonormal system \cite{Candes2005}, hence $\A\x$ is approximately an isometry for $\kappa$-sparse vectors.

Unfortunately, RIP is NP-hard even to verify for a given matrix as it requires ${n \choose \kappa}$ eigendecompositions. The success story lies in that properly constructed random matrices satisfy RIP with overwhelming  probability \cite{Candes2008}, for example: 

\begin{enumerate}
\item Sampling $n$ random vectors uniformly at random from the $m$-dimensional unit sphere~\cite{Candes2008}. 
\item Random partial Fourier matrices obtained by selecting $m$ rows from the $n$ dimensional Fourier matrix $F$ uniformly at random, where:
\begin{align*}
\vec{F} = \frac{1}{\sqrt{n}}\begin{bmatrix}
1 & 1 & \dots & 1\\
\omega & \omega^2 & \dots & \omega^{n-1}\\
\vdots & \vdots & \ddots & \vdots\\
\omega^n & \omega^{2n} & \dots & \omega^{(n-1)^2}
\end{bmatrix}
\end{align*}
for $\omega = e^{i 2\pi/N}$.
\item Random Gaussian matrices with entries drawn i.i.d. from $\mathcal{N}\left(0,1/m\right)$. 
\item Random Bernoulli matrices with
$$A_{i,j} \in \left\{ 1 / \sqrt{m},
-1 / \sqrt{m} \right\}
$$ 
with equal probability, or~\cite{achlioptas}: 
\begin{align}{\label{eq:ach}}
A_{ij} =
\left\{
	\begin{array}{ll} \vspace{0.05in}
		1  & \mbox{with probability } \frac{1}{6}, \\ \vspace{0.05in}
		0  & \mbox{with probability } \frac{2}{3}, \\ 
		-1 & \mbox{with probability } \frac{1}{6}.
	\end{array}
\right.
\end{align} 
having the added benefit of providing a \emph{sparse} sampling matrix.
\end{enumerate}
For the last two cases, $\A$ satisfies a prescribed $\delta_\kappa$ for any $\kappa \leq c_1 m / \log (n/\kappa)$ with probability $p \geq 1 - 2 e^{-c_2m}$, where constants $c_1$ and $c_2$ depend only on $\delta_\kappa$ \cite{baraniuk2008}. The important result here is that such matrix constructions are \emph{universal}--in the sense that they satisfy RIP which is a property that does not depend on the underlying application-- as well as efficient, as they only require random number generation.  

It is typically the case that $\x$ is not itself sparse, but is sparsely representable in a given orthonormal basis. In such case, we write $\x = \vec{\Psi}\vec{\alpha}$, where now $\vec{\alpha}$ is sparse. Compressed sensing then amounts to designing a \emph{sensing matrix} $\vec{\bar{\Phi}}\in \R^{m\times n}$ such that $\A:=\vec{\bar{\Phi}}\vec{\Psi}$ is a CS matrix. Luckily, random matrix constructions can still serve this purpose. 

\begin{mydef}[\textbf{Coherence}]
Let $\vec{\Phi}$, $\vec{\Psi}$ be two orthonormal bases in $\R^n$. The \emph{coherence} between these two bases is defined as~\cite{Candes2008}:
\begin{align}
\mathcal{M} (\vec{\Phi}, \vec{\Psi}) \coloneqq \sqrt{n} \max_{1 \leq k,j \leq n} \left\vert \langle \phi_k, \psi_j \rangle \right\vert.
\end{align}
It follows from elementary linear algebra that $1 \leq \mathcal{M}(\vec{\Phi}, \vec{\Psi}) \leq \sqrt{n}$ for any choice of $\vec{\Phi}$ and $\vec{\Psi}$. 
\end{mydef}

The basis $\vec{\Phi}\in \R^{n\times n}$ is called the \emph{sensing} basis, while $\vec{\Psi}\in\R^{n\times n}$ is the \emph{representation} basis. Compressed sensing results apply for low coherence pairs \cite{Candes2006b}. A typical example of such pairs is the Fourier and canonical basis, for which the coherence is $1$ (\emph{maximal incoherence}). Most notably, a random basis $\vec{\Phi}$ (generated by any of the previously described distributions for $m=n$), when orthonormalized is incoherent with any given basis $\vec{\Psi}$ (with high probability, $\mathcal{M}(\vec{\Phi}, \vec{\Psi}) \approx \sqrt{2 \log n}$) \cite{Candes2008}. The sensing matrix $\vec{\bar{\Phi}}$ can be selected as a row-subset of $\vec{\Phi}$, therefore, designing a sensing matrix is not different than for the case where $\x$ is itself sparse, i.e.,  $\vec{\Psi} = \vec{I}_{n\times n}$, the identity matrix. 
For ease of presentation, we assume in the sequel that $\vec{\Psi} = \vec{I}_{n\times n}$, unless otherwise specified, but the results also hold for the general case.

\subsection{Setting}

Given linear measurements of vector $\x \in \R^{n}$
\begin{align}
\label{eq:noiselessSettings}
\vec{y} = \vec{Ax},
\end{align}
$\vec{y} \in \R^{m}$ is the vector of obtained samples and $\vec{A} \in \R^{m \times n}$ is the \emph{sampling (sensing)} matrix. Our goal is to recover $\vec{x}$ when $m << n$. This is an underdetermined linear system, so for a general vector $\x$ it is essentially  \emph{ill-posed}.  The main result in CS is that if $\vec{x}$ is $\kappa$-sparse and $\kappa < C m / \log (n/k)$, this is \emph{possible}. Equivalently, random linear measurements can be used for compressively \textbf{encoding} a sparse signal, so that it is then possible to reconstruct it in a subsequent decoding step. The key concept here is that encoding is \emph{universal}: while it is straightforward to compress if one knows the positions of the non-zero entries, the CS approach works for \emph{all} sparse vectors, without requiring prior knowledge of the non-zero positions. 

To this end, searching for the sparsest vector $\vec{x}$ that leads to the measurement $\vec{y}$, one needs to solve
\begin{equation}
\begin{aligned}
& {\text{min}}
& & \Vert \vec{x} \Vert_0\\
& \text{s.t.}
& & \vec{Ax} = \vec{y}.
\end{aligned}
\label{eq:zeroNorm}
\tag{$P_0$}
\end{equation}
Unfortunately this problem is, in general, NP-hard requiring search over all subsets of columns of $A$ \cite{baraniuk2008}, e.g., checking $n \choose \kappa$ linear systems for a solution in the worst case.

\subsection{Algorithms for Sparse Recovery}
\label{sec:algorithms}

Since \eqref{eq:zeroNorm} may be computationally intractable for large instances, one can seek to `approximate' it by other tractable methods. In this section, we summarize several algorithms used for recovering sparse vectors from linear measurements, at the \textbf{decoding} phase, with provable performance guarantees.


\subsubsection{Basis Pursuit}

Cand\`{e}s and Tao \cite{Candes2005} have shown that solving \eqref{eq:zeroNorm} is equivalent to solving the $\ell_1$ minimization problem
\begin{equation}
\begin{aligned}
& {\text{min}}
& & \Vert \vec{x} \Vert_1 \\
& \text{s.t.}
& & \vec{Ax} = \vec{y},
\end{aligned}
\label{eq:basisPursuit}
\tag{$BP$}
\end{equation}
for all $\kappa$-sparse vectors $\x$, if $\vec{A}$ satisfies RIP with $\delta_{2\kappa} < \sqrt{2} -1$. The optimization problem \eqref{eq:basisPursuit} is called \textit{Basis Pursuit}. Since the problem can be recast as a linear program, solving \eqref{eq:basisPursuit} is computationally efficient, e.g., via interior-point methods \cite{bertsekas2009convex}, even for large problem instances as opposed to solving~\eqref{eq:zeroNorm} whose computational complexity may be prohibitive. 

\subsubsection{Orthogonal Matching Pursuit}
Orthogonal Matching Pursuit (OMP) is a greedy algorithm that seeks to recover sparse vectors $\x$ from noiseless measurement $\y = \A\x$. The algorithm outputs a subset of columns of A, via iteratively selecting the column minimizing the residual error of approximating $\y$ by projecting to the linear span of previously selected columns.

Assuming $\x$ is $\kappa$-sparse, the resulting measurement $\y$ can be represented as the sum of at most $\kappa$ columns of $\A$ weighted by the corresponding nonzero entries of $\x$. Let the columns of $\A$ be normalized to have unit $\ell_2$-norm. For iteration index $t$, let $\vec{r}_t$ denote the residual vector, let $\vec{c}_t \in \R^t$ be the solution to the least squares problem at iteration $t$, $S_t$ set of indices and $\A_{S_t}$ the submatrix obtained by extracting columns of $\A$ indexed by $S_t$. 
The OMP algorithm operates as follows: Initially set $t=1$, $\vec{r}_0 = \y$ and $S_0 = \varnothing$. At each iteration the index of the column of $\A$ having highest inner product with the residual vector, i.e., $s_t = \arg \max_i \langle \vec{r}_{t-1}, \vec{a}_i \rangle$ is added to the index set, yielding $ S_{t} = S_{t-1} \cup \{s_t\}$. Since one index is added to $S_t$ at each iteration the cardinality is $\vert S_t \vert = t$. Then, the least squares problem 
\begin{align*}
\vec{c}_t = \arg \min_{\vec{c} \in \mathbb{R}^t} \: \Vert \y - \sum\limits_{j=1}^{t} c_{j}a_{s_j} \Vert_2
\end{align*}
is solved in each iteration; 
a closed form solution is:
\begin{align*}
\vec{c}_t = \left( \A_{S_t}^{\T}\A_{S_t} \right)^{-1} \A_{S_t}^{\T} \y,
\end{align*}
and the residual vector is updated by $\vec{r}_t = \y - \sum\limits_{j=1}^t c_{tj}a_{s_j}$. With $\vec{r}_t$ obtained as such, the residual vector at the end of iteration $t$ is made orthogonal to all the vectors in the set $\left\{\vec{a}_i: i \in S_t\right\}$. 

The algorithm stops when a desired stopping criterion is met, such as $\Vert \y - \A_{S_t} \vec{c}_t \Vert_2 \leq \gamma$ for some threshold $\gamma \geq 0$. 
Despite its simplicity, there are guarantees for \emph{exact recovery}; OMP recovers \emph{any} $\kappa$-sparse signal exactly if the mutual coherence of the measurement matrix $\A$ satisfies $\mu(\A) < \frac{1}{2 \kappa-1}$ \cite{Tropp2004}. 
\hide{
\begin{algorithm}[!ht]
\caption{Orthogonal Matching Pursuit}
\begin{algorithmic}[1]
\Require $\A \in \mathbb{R}^{m \times n}$, $\y \in \mathbb{R}^{m}$, threshold $\gamma \geq 0$
\Ensure sparse vector $\hat{\x}$.
\State $\vec{r}_0 \gets \y, \: S_0 \gets \varnothing, \: t \gets 1$
\For{$t = 1,2,3,\dots$}
\State $s_t \gets \text{arg\,}\underset{i \not\in S_{t-1}}{\text{min}} \langle \vec{r}_{t-1}, \vec{a}_i \rangle$
\State $ S_{t} \gets S_{t-1} \cup \{s_t\}$
\State $\vec{c}_t = \text{arg\,}\underset{c \in \mathbb{R}^t}{\text{min}} \: \Vert y - \sum\limits_{j=1}^{t} c_{j}a_{s_j} \Vert_2$
\State $\vec{r}_t \gets \y - \sum\limits_{j=1}^{t} c_{tj}a_{s_j}$
\State Check stopping criterion and terminate if it holds: 
\NoNumber{\qquad $\Vert \y - \A_{S_t} \vec{c}_t \Vert_2 \leq \gamma$}
\EndFor
\State Obtain $\hat{\x}$ by 
\begin{align}\label{eq:obtaining_x}
\hat{x}_i = \begin{cases}
c_j &\mbox{if } i = s_j \text{ for } j \in \left\{1,2,\dots,t \right\}\\
0 &\mbox{else }
\end{cases}
\end{align}

\end{algorithmic}

\label{alg:OMP}
\end{algorithm}
}
Both BP and OMP handle the case of \emph{noiseless} measurements. However, in most practical scenaria, noisy measurements are inevitable, and we address this next. 
\subsubsection{Least Absolute Selection and Shrinkage Operator (LASSO)}

Given measurements of vector $\x \in \R^n$ corrupted by additive noise:
\begin{align}
\vec{y} = \vec{Ax} + \vec{w},
\label{eq:nm}
\end{align}
one can solve a relaxed version of BP, where the equality constraint is replaced by inequality to account for measurement noise:
\begin{equation}
\begin{aligned}
& {\text{min}}
& & \Vert \vec{x} \Vert_1 \\
& \text{s.t.}
& & \Vert \A\x - \y \Vert_2 \leq \tilde{\sigma},
\end{aligned}
\label{eq:lasso_form2}
\end{equation}
This is best known as Least Absolute Selection and Shrinkage Operator (LASSO) in the statistics literature \cite{Tibshirani1996}. The value of $\tilde{\sigma}$ is selected to satisfy $\tilde{\sigma} \geq \Vert \w \Vert_2$.

By duality, the problem can be posed equivalently~\cite{bertsekas2009convex} as an unconstrained $\ell_1$-regularized least squares problem: 
\begin{align}
\mbox{min}& \Vert \vec{Ax}-\vec{y} \Vert_2^2 + \lambda\Vert \vec{x} \Vert_1,
\label{eq:lassounconstrained}
\end{align}
where $\lambda$ is the regularization parameter that controls the trade-off between sparsity and reconstruction error.
Still by duality, an equivalent version is given by the following constrained optimization problem: 
\begin{equation}
\begin{aligned}
& {\text{minimize}}
& & \Vert \A\x - \y \Vert_2 \\
& \text{subject to}
& & \Vert \x \Vert_1 \leq \mu.
\end{aligned}
\label{eq:lasso_form3}
\end{equation}
\begin{rem}[Equivalent forms of LASSO]
All these problems can be made equivalent--in the sense of having the same solution set--for particular selection of parameters  $(\tilde{\sigma},\lambda,\mu)$. This can be seen by casting between the optimality conditions for each problem; unfortunately the relations obtained depend on the optimal solution itself, so there is no analytic formula for selecting a parameter from the tuple $(\tilde{\sigma},\lambda,\mu)$ given another one. In the sequel, we refer to both \eqref{eq:lasso_form2}, \eqref{eq:lassounconstrained} as LASSO; the distinction will be made clear from the context.
\end{rem}

The following theorem characterizes recovery accuracy in the noisy case through LASSO.
\begin{theorem}[Error of LASSO \cite{Candes2008}] \label{thm:LASSO_error}
If $\A$ satisfies RIP with $\delta_{2\kappa} < \sqrt{2} -1$, the solution $\x_*$ to \eqref{eq:lasso_form2} obeys:
\begin{equation}
\begin{aligned}
\Vert \x_* - \x \Vert_2 \leq c_0 \cdot \Vert \x - \x_\kappa \Vert_1 / \sqrt{\kappa} + c_1 \cdot \tilde{\sigma},
\end{aligned}
\end{equation}
for constants $c_0$ and $c_1$, where $\x_\kappa$ is the vector $\x$ with all but the largest $\kappa$ components set to 0. 
\label{th:candes}
\end{theorem}
Theorem \ref{th:candes} states that the reconstruction error is upper bounded by the sum of two terms: the first is the error due to \emph{model mismatch}, and the second is proportional to the \emph{measurement noise variance}. In particular, if $\x$ is $\kappa$-sparse and $\delta_{2\kappa} < \sqrt{2} -1$ then $\Vert \x_* - \x \Vert_2 \leq c_1 \cdot \tilde{\sigma}$. Additionally, for noiseless measurements $\w=\mathbf{0} \implies \tilde{\sigma}=0$, we retrieve the success of BP as a special case (note that the requirement on the restricted isometry constant is identical). This assumption is satisfied with high probability by matrices obtained from random vectors sampled from the unit sphere, random Gaussian matrices and random Bernoulli matrices if $m \geq C \kappa \log (n / \kappa)$, where $C$ is a constant depending on each instance \cite{Candes2008}; typical values for the constants $C_0$ and $C_1$ can be found in \cite{Candes2008} and \cite{Candes2006b}, where it is proven that $C_0 \leq 5.5$ and $C_1 \leq 6$ for $\delta_{2k} = 1/4$. A different approach was taken in~\cite{goyal}, where the replica method was used for analyzing the mean-squared error of LASSO.

The difficulty of solving \eqref{eq:zeroNorm} lies in estimating the \emph{support} of vector $\x$, i.e., the positions of the non-zero entries. One may assume that solving LASSO may give some information  on support, and this is indeed the case \cite{candes2009}. To state the result on support detection we define the \emph{generic} $\kappa$-sparse model.

\begin{mydef}[Generic $\kappa$-sparse model]
Let $\x \in \R^n$ denote a $\kappa$-sparse signal and $I_{\x} :=  \mbox{supp}(\x)$ be its support set, $\mbox{supp}(\x) := \{i: x_i \ne 0\}$. Signal $\x$ is said to be generated by \emph{generic $\kappa$-sparse model} if:
\vspace{-1.2mm}
\begin{enumerate}
\item Support $I_{\x} \subset \{1,2,\dots,n\}$ of $\x$ is selected uniformly at random, and $\vert I_x \vert = \kappa$.
\item Conditioned on $I_{\x}$, the signs of the non zero elements are independent and equally likely to be $-1$ and $1$.
\end{enumerate}
\end{mydef}
\begin{theorem}[Support Detection \cite{candes2009}]
 Assume $\mu(\A) \leq c_1 / \log n$ for some constant $c_1 > 0$, $\x$ is generated from \emph{generic $\kappa$-sparse model}, $\kappa \leq c_2 n / (\Vert \A \Vert_2^2 \log n)$ for some constant $c_2>0$ and $\w \sim \mathcal{N}(0,\sigma^2 \vec{I})$.  If $\underset{i\in I_{\x}}{\min} \; \vert x_i \vert > 8 \sigma \sqrt{2 \log n}$, the LASSO estimate obtained by choosing $\lambda = 4 \sigma\sqrt{2 \log n}$ satisfies:
\begin{align*}
\supp (\hat{\x}) &= \supp (\x) \\
\sgn (\hat{x}_i) &= \sgn (x_i), \ \ \forall i \in I_{\x}
\end{align*} 
with probability at least $1 - \frac{2}{n} \left( \frac{1}{\sqrt{2\pi \log n}} + \frac{\vert I_{\x} \vert}{n}\right) - O\left( \frac{1}{n^{2\log 2}}\right)$ and
\begin{align*}
\Vert \vec{A}\vec{x}-\vec{A}\hat{\vec{x}} \Vert^2_2\leq c_3 \kappa (\log n)\sigma^2,
\end{align*}
with probability at least $1-6n^{-2\log 2}-n^{-1} (2\pi \log n)^{-1/2}$, for some positive constant $c_3$.
\label{theorem:modelselection}
\end{theorem}
Another result on near support detection, or as alternatively called \emph{ideal model selection} for LASSO is given in  \cite{Zhang2008} based on the so called \emph{irrepresentable} condition of the sampling matrix introduced therein.


\begin{rem}[Algorithms for LASSO] 
There is a wealth of numerical methods for LASSO stemming from convex optimization. 
LASSO is a convex program (in all equivalent forms) and the unconstrained problem \eqref{eq:lassounconstrained} can be easily recast as a quadratic program, which can be handled by interior point methods~\cite{Boyd2004}. This is the case when using a generalized convex solver such as cvx \cite{cvx}. 
Additionally, iterative algorithms have been developed specifically for LASSO; all these are inspired by proximal methods  \cite{Parikh:2013} for non-smooth convex optimization: FISTA \cite{Beck2009} and SpaRSA \cite{Wright:2009kw} are accelerated proximal gradient methods \cite{Parikh:2013}, SALSA \cite{Afonso2010} is an application of the alternative direction method of multipliers. These methods are \emph{first-order} methods~\cite{Boyd2004}, in essence generalizations of the gradient method.  
For error defined as $G(\x_{[t]}) - G(\x_*)$ where $G(\x)$ is the objective function of LASSO in \eqref{eq:lassounconstrained}, $\x_{[t]}$ is the estimate at iteration number $t$ and $\x_* = \argmin_{\x} G(\x)$ is the optimal soultion, the error decays as $1/t^2$ for FISTA, SpaRSA and SALSA. Recently, a proximal Newton-type method was devised for LASSO \cite{panos_nick} with substantial speedup; the convergence rate is globally no worse than $1/t^2$, but is locally \emph{quadratic} (i.e., goes to zero roughly like $e^{-ct^2}$).
\end{rem}
\begin{rem} [Computational complexity]
In iterative schemes, computational complexity is considered at a per-iteration basis: a) interior-point methods require solving a \emph{dense} linear system, hence a cost of $O(n^3)$ per iteration, b) first-order proximal methods only perform matrix-vector multiplications at a cost of $O(n^2)$, while the second-order method proposed in~\cite{panos_nick} requires solving a \emph{sparse} linear system at a resulting cost of $O(\kappa^3)$. The total complexity depends also on the number of iterations until convergence; we analyze this in Sec.~\ref{sec:comp_complexity}. Note that the cost of decoding dominates that of encoding which requires a single matrix-vector multiplication, i.e., $O(mn)$ operations.

\end{rem}

Our approach is generic, in that it does not rely on a particular selection of numerical solver. It uses \emph{warm-start} for accelerated convergence, so using an algorithm like \cite{panos_nick} may yield improvements over the popular FISTA that we currently use in experiments.

We conclude this section by providing optimality conditions for LASSO, which can serve in determining termination criteria for iterative optimization algorithms. We show the case of unconstrained LASSO, but similar conditions hold for the constrained versions \eqref{eq:lasso_form2},  \eqref{eq:lasso_form3}.
\begin{rem}[Optimality conditions for LASSO] 
For unconstrained LASSO cf. \eqref{eq:lassounconstrained}, define $\x^*$  to be an optimal solution, and $\vec{g} := \A^{\T} \left( \y - \A\x^*  \right)$. The necessary and sufficient KKT conditions for optimality \cite{Boyd2004} are: 
\begin{equation}
\begin{aligned}
g_i &= \frac{\lambda}{2} \,\sgn (x_i^*)  && \mbox{for} \left\{i: x_i^* \neq 0 \right\}\\
\vert g_j \vert &< \frac{\lambda}{2} && \mbox{for} \left\{j: x_j^* = 0 \right\}.
\label{eq:lassooptimality}
\end{aligned}
\end{equation}
\end{rem}
\noindent As termination criterion, we use $\epsilon$-optimality, for some $\epsilon>0$ suffieciently small:
\begin{equation}
\begin{aligned}
\vert g_i - \frac{\lambda}{2} \,\sgn (x_i^*)\vert &\leq \epsilon && \mbox{for} \left\{i: x_i^* \neq 0 \right\}\\
\vert g_j \vert &< \frac{\lambda}{2} + \epsilon && \mbox{for} \left\{j: x_j^* = 0 \right\}.
\end{aligned}
\end{equation}

\section{Recursive Compressed Sensing}
\label{sec:recursivecs}
We consider the case that the signal of interest is an infinite sequence, $\{ x_i\}_{i = 0,1,\dots}$, and process the input stream via successive windowing; 
we define 
\begin{align}
\vec{x}^{(i)} := \begin{bmatrix}
x_{i} &x_{i+1} & \dots &x_{i+n-1}
\end{bmatrix}^{\T}
\label{eq:wstructure}
\end{align}
to be the $i^{th}$ window taken from the streaming signal. If $\x^{(i)}$ is known to be sparse, one can apply the tools surveyed in Section~\ref{sec:background} to recover the signal portion in each window, hence the data stream. However, the involved operations are costly and confine an efficient online implementation. 

In this section, we present our approach to compressively sampling streaming data, based on recursive \textbf{encoding-decoding}. The proposed method has low complexity in both the sampling and estimation parts which makes the algorithm suitable for an online implementation.

\subsection{Problem Formulation}

From the definition of $\x^{(i)} \in \R^n$ we have:
\begin{align}
\vec{x}^{(i)} = 
\begin{bmatrix}
0 & 1 & 0 &\dots & 0\\
0 & 0 & 1 &\dots & 0\\
\vdots & \vdots & \vdots & \ddots & \vdots \\
0 & 0 & 0 &\dots & 1 \\
0 & 0 & 0 &\dots & 0
\end{bmatrix} \x^{(i-1)} +
\begin{bmatrix}
0 \\ 0 \\ \vdots \\ 0 \\ 1
\end{bmatrix} x_{i+n-1},
\label{eq:window_structure}
\end{align}
which is in the form of a $n-$dynamical system with scalar input. The sliding window approach is illustrated in Fig.~\ref{fig:signal_stream}.

\begin{figure}[t]
\centering
\includegraphics[width = 0.6\linewidth]{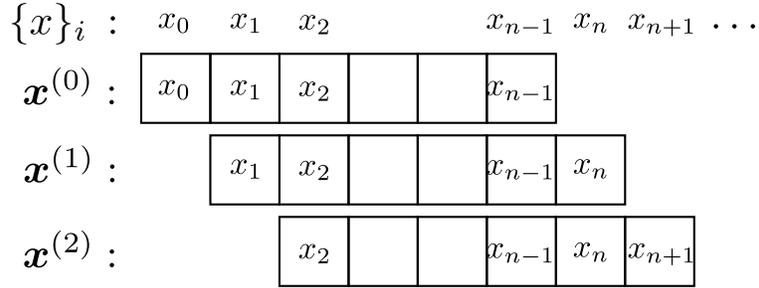}
\caption{Illustration of the overlapping window processing for the data stream $\{ \x^{(i)} \}_{i=0,1\dots}$.}
\label{fig:signal_stream}
\end{figure}
  
Our goal is to design a robust low-complexity sliding-window algorithm which provides estimates $\{\hat{x}_i \}$ using successive measurements $\y^{(i)}$ of the form
\begin{align}
\label{eq:problemFormulation}
\vec{y}^{(i)} = \vec{A}^{(i)}\vec{x}^{(i)} + \w^{(i)},
\end{align}
where $\{ \A^{(i)} \}$ is a sequence of measurement matrices. This is possible if $\{ x_i\}$ is sufficiently sparse in each window, namely if $\Vert \x^{(i)} \Vert_0  \leq \kappa$ for each $i$, where $\kappa << n$ (or if this holds with sufficiently high probability), and $\{\A^{(i)}\}$ are CS matrices, i.e., satisfy the RIP as explained in the prequel.

Note that running such an algorithm online is costly and therefore, it is integral to design an alternative to an ad-hoc method. We propose an approach that leverages the signal overlap between successive windows, consisting of recursive sampling and recursive estimation.

{\bf Recursive Sampling:} To avoid a full matrix-vector  multiplication for each $\vec{y}^{(i)}$, we design $\A^{(i)}$ so that we can reuse $\vec{y}^{(i)}$ in computing $\vec{y}^{(i+1)}$ with low computation overhead, or
\begin{align*}
\vec{y}^{(i+1)} = f\left(\vec{y}^{(i)}, x_{i+n}, x_i \right).
\end{align*}

{\bf Recursive Estimation:} In order to speed up the convergence of an iterative optimization scheme, we make use of the estimate corresponding to the previous window, $\vec{\hat{x}}^{(i-1)}$, to derive a starting point, $\vec{\hat{x}}^{(i)}_{[0]}$, for estimating  $\vec{\hat{x}}^{(i)}$, or 
\begin{align*}
\vec{\hat{x}}^{(i)}_{[0]} = g\left(\vec{\hat{x}}^{(i-1)}, \vec{y}^{(i)} \right).
\end{align*}

\subsection{Recursive Sampling of sparse signals}
\label{sec:recursivesampling}

We propose the following recursive sampling scheme with low computational overhead; it reduces the complexity to $O(m)$ vs. $O(mn)$ as required by the standard data encoding. 

We derive our scheme for the most general case of noisy measurements, with ideal measurements following as special case.

At the first iteration, there is no prior estimate, so we necessarily have to compute
\begin{align*}
\vec{y}^{(0)} = \vec{A}^{(0)}\vec{x}^{(0)} + \w^{(0)}.
\end{align*}
We choose a sequence of sensing matrices $\A^{(i)}$ recursively as:
\begin{align}
\label{eq:updateA}
\A^{(i+1)}
&= \begin{bmatrix}
\vec{a}_1^{(i)} & \vec{a}_2^{(i)} & \dots & \vec{a}_{n-1}^{(i)} & \vec{a}_0^{(i)}
\end{bmatrix} = \A^{(i)} \vec{P}
\end{align}
where $\vec{a}_j^{(i)}$ is the $j^{th}$ column of $\vec{A}^{(i)}$--where we have used the convention $j \in \{0,1,\dots,n-1\}$ for notational convenience--and $\vec{P}$ is a permutation matrix:
\begin{align}
\vec{P} \coloneqq \begin{bmatrix}
0   &\dots & 0 & 1\\
1   &\dots & 0 & 0\\
\vdots & \ddots & \vdots & \vdots \\
0 & \dots & 1 & 0
\end{bmatrix}.
\label{eq:p}
\end{align}
The success of this data encoding scheme is ensured by noting that 
if $\vec{A}^{(0)}$ satisfies RIP for given $\kappa$ with constant $\delta_{\kappa}$, then $\vec{A}^{(i)}$ satisfies RIP for the same $\kappa$, $\delta_{\kappa}$, due to the fact that RIP is insensitive to permutations of the columns of $\A^{(0)}$.

%

Given the particular recursive selection of $\A^{(i)}$ we can compute $\vec{y}^{(i+1)}$ recursively as:
\begin{align}
\vec{y}^{(i+1)} &= \A^{(i+1)} \x^{(i+1)} + \w^{(i+1)} \nonumber \\
&= \A^{(i)} \vec{P} \x^{(i+1)} + \w^{(i+1)} \nonumber\\
&= \A^{(i)} \left( \x^{(i)} + \begin{bmatrix} 1 \\ \vec{0}_{n-1} \end{bmatrix} (x_{i+n} - x_i) \right) + \w^{(i+1)} \nonumber\\
&= \vec{y}^{(i)} + (x_{i+n}- x_{i})\vec{a}^{(i)}_1 + \w^{(i+1)} - \w^{(i)},
\label{eq:noNoise}
\end{align}
%
where $\vec{0}_{n-1}$ denotes the all 0 vector of length $n-1$.
This takes the form of a noisy \emph{rank-$1$} update:
\begin{align}
 \vec{y}^{(i+1)} = \vec{y}^{(i)} + \underbrace{(x_{i+n}- x_{i})\vec{a}^{(i)}_1}_{\mbox{rank-1 update}} + \vec{z}^{(i+1)},
 \label{eq:rank1}
 \end{align}
where the \emph{innovation} is the scalar difference between the new sampled portion of the stream, namely $x_{i+n}$, and  the entry $x_{i}$ that belongs in the previous window but not in the current one. Above, we also defined $\vec{z}^{(i)} := \w^{(i)} - \w^{(i-1)}$ to be the noise increment; note that the noise sequence $\{\vec{z}^{(i)}\}$ has independent entries if $\w^{(i)}$ is an independent increment process. Our approach naturally extends to sliding the window by $1<\tau\le n$ units, in which case we have a rank-$\tau$ update, cf. Sec. \ref{sec:discussion}

\begin{rem}
The particular selection of the sampling matrices $\{\A^{(i)}\}_{i=0,1,\dots}$ given in \eqref{eq:updateA} satisfies $\A^{(i)} \x^{(i)} = \A^{(0)} \vec{P}^{i} \x^{(i)}$. Defining 
\begin{align}
	\vec{v}^{(i)} \coloneqq \vec{P}^i \x^{(i)},
	\label{eq:v}
\end{align}
recursive sampling can be viewed as encoding $\vec{v}^{(i)}$ by using the same measurement matrix $\A^{(0)}$. With the particular structure of $\x^{(i)}$ given in \eqref{eq:wstructure}, all of the entries of $\vec{v}^{(i)}$ and $\vec{v}^{(i-1)}$ are equal except $v^{(i)}_{i-1}$. 
Thus the resulting problem can be viewed as signal estimation with partial information.
\end{rem}

\begin{figure}[t]
\centering
\includegraphics[totalheight=0.3\textheight]{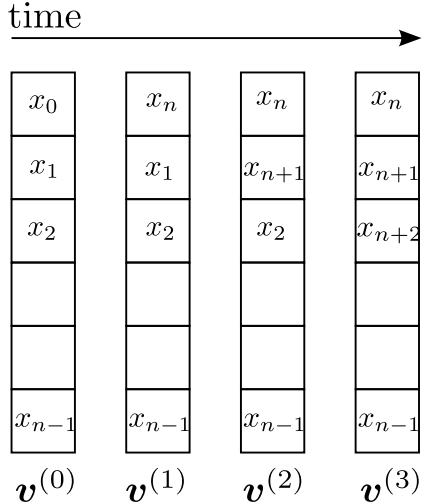}
\caption{Illustration of $\vec{v}^{(i)}$ for the first four windows. The sampling of $\x^{(i)}$ by the matrix $\A^{(i)}$ is equivalent to sampling $\vec{v}^{(i)}$ by $\A^{(0)}$, i.e., $\A^{(i)} \x^{(i)} = \A^{(0)} \vec{v}^{(i)}$.}
\label{fig:vector_v}
\end{figure}

\subsubsection{Recursive sampling in Orthonormal Basis}

So far, we have addressed the case that for a given $n \in \Z^+$, a given window $\x^{(i)}$ of length $n$ obtained from the sequence $\{ x_i \}$ is $\kappa$-spase: $\Vert \x^{(i)} \Vert_0 \leq \kappa$, $\forall i$. In general, it might rarely be the case that $\x^{(i)}$ is sparse itself, however it may be sparse when represented in a properly selected basis (for instance the Fourier basis for time series or a wavelet basis for images). We show the generalization below.

Let $\x^{(i)} \in \R^n$ be sparsely representable in a given orthonormal basis $\vec{\Psi}$, i.e., $\x^{(i)} = \vec{\Psi} \vec{\alpha}^{(i)}$, where $\vec{\alpha}^{(i)}$ is sparse. Assuming a common basis for the entire sequence $\{ x_i\}$ (over windows of size n) we have:
\begin{align*}
\vec{A}^{(i)}\vec{x}^{(i)} =  \vec{A}^{(i)}\vec{\Psi}\vec{\alpha}^{(i)}.
\end{align*}
For the CS encoding/decoding procedure to carry over, we need that $\vec{A}^{(i)}\vec{\Psi}$ satisfy RIP. 
The key result here is that RIP is satisfied with high probability for the product of a random matrix $\vec{A}^{(i)}$ and any fixed matrix \cite{baraniuk2008}. In this case the LASSO problem is expressed as:
\begin{align*}
\mbox{minimize~~~}& \Vert \A^{(i)}\vec{\Psi}\vec{\alpha}^{(i)} -\vec{y}^{(i)} \Vert_2^2 + \lambda\Vert \vec{\alpha}^{(i)} \Vert_1,
\end{align*}
where the input signal is expressed as $\x^{(i)} = \vec{\Psi} \vec{\alpha}^{(i)}$, and measurements are still given by $\y^{(i)} = \A^{(i)}\x^{(i)} + \w^{(i)}$.

\begin{lemma}[Recursive Sampling in Orthonormal Basis]
Let $\x^{(i)} = \vec{\Psi} \alpha^{(i)}$, where $\vec{\Psi}$ is an orthonormal matrix with inverse $\vec{\Gamma}:=\vec{\Psi}^{-1}$.  Then, 
\begin{align}
\vec{\alpha}^{(i+1)} = &\vec{\Gamma} \vec{\Pi} \vec{\Psi} \vec{\alpha}^{(i)} \notag 
+ \vec{\gamma}_{n-1} \left( \vec{\psi}_{(n-1)}  \vec{\alpha}^{(i + 1)} - \vec{\psi}_{(0)}  \vec{\alpha}^{(i)} \right),
\end{align}
where $\vec{\Pi} \coloneqq \vec{P}^\T$, and $\vec{\gamma}_{(0)}$ and $\vec{\gamma}_{(n-1)}$ denote the first and last row of $\vec{\Gamma}$, respectively.
\end{lemma}
\begin{proof}
By the definition of $\x^{(i+1)}$ we have:
\begin{align*}
\x^{(i+1)} &= \vec{\Pi} \x^{(i)} + 
\begin{bmatrix}
\vec{0}_{n-1} \\
1
\end{bmatrix}
(x_{i+n}- x_{i}).
\end{align*} 
Since $\x^{(i)} = \vec{\Psi} \vec{\alpha}^{(i)}$, it holds:
\begin{align*}
x_{i} = x^{(i)}_{0} &= \begin{bmatrix}1 & \vec{0}_{n-1} \end{bmatrix} \vec{\Psi} \vec{\alpha}^{(i+1)}\\
x_{i+n} = x^{(i+1)}_{n-1} &= \begin{bmatrix}\vec{0}_{n-1} & 1 \end{bmatrix} \vec{\Psi} \vec{\alpha}^{(i+1)}.
\end{align*}
Using these equations along with $\vec{\alpha}^{(i)}  = \vec{\Gamma}\x^{(i)}$ yields:
\begin{align*}
\vec{\alpha}^{(i+1)} &= \vec{\Gamma} \x^{(i+1)} = \vec{\Gamma} \vec{\Pi} \x^{(i)} + \vec{\Gamma}
\begin{bmatrix}
\vec{0}_{n-1} \notag\\
1
\end{bmatrix}
(x_{i+n}- x_{i}) \notag\\
&= \vec{\Gamma} \vec{\Pi} \x^{(i)} + 
(x_{i+n}- x_{i}) \vec{\gamma}_{n-1} \notag\\
&= \vec{\Gamma} \vec{\Pi} \vec{\Psi} \vec{\alpha}^{(i)} + \vec{\gamma}_{n-1} \left( (\vec{\psi}^\T_{n-1})^\T  \vec{\alpha}^{(i + 1)} - (\vec{\psi}^\T_{0})^\T  \vec{\alpha}^{(i)} \right).
\end{align*}
\end{proof}
\vspace{-0.5cm}
\subsubsection{Recursive Sampling in Fourier Basis}

The Fourier basis is of particular interest in many practical applications, e.g., time-series analysis.  For such a basis, an efficient update rule can be derived, as is shown in the next corollary.

\begin{corollary}
Let $\vec{\Psi}$ be $n \times n$ inverse Discrete Fourier Transform (IDFT) matrix with entries $\Psi_{i,j} = \omega^{ij} / \sqrt{n}$ where $i,j\in \{0,\dots,n-1\}$ and $\omega := e^{j\frac{2\pi}{n}}$. In such case:
\begin{align}
\vec{\alpha}^{(i+1)} &= 
\vec{\Omega}_n \vec{\alpha}^{(i)} + \vec{f}_{n-1} \left( \vec{\psi}_{(n-1)}  \vec{\alpha}^{(i + 1)} - \vec{\psi}_{(0)}  \vec{\alpha}^{(i)} \right)
\label{eq:updateRuleFourierResult}
\end{align}
where $\vec{\Omega}_n$ is the $n \times n$ diagonal matrix with $(\Omega_n)_{i,i} = \omega^{-i}$,
and $\vec{F} = \vec{\Psi}^{-1}$ is the orthonormal Fourier basis.
\label{lemma:fourier}
\end{corollary}
\begin{proof}
Circular shift in the time domain corresponds to multiplication by complex exponentials in the Fourier domain, i.e., $\vec{F} \vec{\Pi} = \vec{\Omega}_n \Vec{F}$, and the result follows from $\Vec{F} \vec{\Psi} = \vec{I}$.
\end{proof}

%
%
%

\begin{rem}[Complexity of recursive sampling in an orthonormal basis]
In general, the number of computations for calculating $\vec{\alpha}^{(i+1)}$ from $\vec{\alpha}^{(i)}$ is $O(n^2)$. 
For the particular case of using Fourier basis, the complexity is reduced to only $O(n)$, i.e., we have \emph{zero-overhead} for sampling directly on the Fourier domain. 
\end{rem}
\subsection{Recursive Estimation}

In the absence of noise, estimation is trivial, in that it amounts to successfully decoding the first window $\xi^{(i)}$, e.g., by BP; then all subsequent stream entries can be plainly retrieved by solving a redundant consistent set of linear equations $\vec{y}^{(i+1)} = \vec{y}^{(i)} + (x_{i+n}- x_{i})\vec{a}^{(i)}_1$ where the only unknown is $x_{i+n}$. 
For noisy measurements, however, this approach is not a valid option due to error propagation: it is no longer true that  $\vec{\hat{x}}^{(i)} = \vec{x}^{(i)}$, so computing $x_{i+n}$ via~\eqref{eq:noNoise} leads to accumulated errors and poor performance.

For \emph{recursive estimation} we seek to find an estimate $\vec{\hat{x}}^{(i+1)} =  \begin{bmatrix} \hat{x}_{0}^{(i+1)} \dots \hat{x}_{n-1}^{(i+1)} \end{bmatrix}$ leveraging the estimate $\vec{\hat{x}}^{(i)} =  \begin{bmatrix} \hat{x}_{0}^{(i)} \dots \hat{x}_{n-1}^{(i)} \end{bmatrix}$ and using LASSO
$$
\begin{aligned}
\hat{x}^{(i+1)} &= \underset{x}{\text{arg\;min}}
& & \Vert \A^{(i+1)} \x-\vec{y}^{(i+1)} \Vert_2^2 + \lambda\Vert \vec{x} \Vert_1
\end{aligned}.
\label{eq:noisycase1}
$$

In iterative schemes for convex optimization, convergence speed depends on the distance of the starting point to the optimal solution \cite{Bertsekas1995}. In order to accelerate convergence, we leverage the overlap between windows and set the starting point as:
\begin{align*}
\hat{\vec{x}}_{[0]}^{(i)} 
&= \begin{bmatrix}
\hat{x}_1^{(i-1)} & \hat{x}_2^{(i-1)} & \dots & \hat{x}_{n-1}^{(i-1)} & * 
\end{bmatrix}^\T,
\end{align*}
where $\hat{x}^{(i-1)}_j$, for $j=1,\dots,n-1$, is the portion of the optimal solution based on the previous window; we set $\hat{x}^{(i-1)}_j, j=0,1,\hdots,n-1$ to be the estimate of the $(j+1)$-th entry of the previous window, i.e., of $x_{i-1+j}$. 
The last entry $\hat{x}^{(i)}_{n-1}$ (denoted by ``*'' above) can be selected using prior information on the data source; for example, for randomly generated sequence, the maximum likelihood estimate $\evc{\x^{(i-1)}}{x_{i+n-1}}$ may be a reasonable option, or we can simply set $\hat{x}^{(i)}_{n-1}=0$, given that the sequence is assumed sparse. By choosing the starting point as such, the expected number of iterations for convergence is reduced (cf. Section \ref{sec:analysis} for a quantitative analysis).

In the general case where the signal is sparsely-representable in an orthonormal basis, one can leverage the recursive update for $\vec{\alpha}^{(i+1)}$ (based on $\vec{\alpha}^{(i)}$) so as to acquire an initial estimate for warm start in recursive estimation, e.g., $\ev{\vec{\alpha}^{(i+1)} \vert \vec{\alpha}^{(i)}}$.

\subsection{Averaging LASSO Estimates}

One way to enhance estimation accuracy, i.e., to reduce estimation error variance, is to average the estimates obtained from successive windows. In particular, for the $i^{th}$ entry of the streaming signal, $x_i$, we may obtain an estimate by averaging\footnote{For notational simplicity, we consider the case $i \geq n - 1$, whence each entry $i$ is included in exactly $n$ overlapping windows. The case $i<n-1$ can be handled analogously by considering $i+1$ estimates instead.} the values corresponding to $x_i$ obtained from \emph{all} windows that contain the value, i.e., 
$\{ \hat{\x}^{(j)} \}_{j= i-n+1,\dots,i}$: 

\begin{equation}
\bar{x}_i := \frac{1}{n} \sum\limits_{j = i - n + 1 }^{i} \hat{x}^{(j)}_{i-j}.
\label{eq:averaged_estimate} 
\end{equation}
By Jensen's inequality, we get:
\begin{align*}
\frac{1}{n} \sum_{j = i - n + 1}^{i} 
\left( \hat{x}^{(j)}_{i-j}  - x_i\right)^2 &\geq 
\left( \frac{1}{n} \sum_{j = i - n + 1}^{i} 
\left( \hat{x}^{(j)}_{i-j}  - x_i\right) \right)^2 \\
&= \left(  \bar{x}_i - x_i \right)^2,
\end{align*}
which implies that averaging may only decrease the reconstruction error--defined in the $\ell_2$-sense. In the following, we analyze the expected $\ell_2$-norm of the reconstruction error $\left( \bar{x}_i -x_i \right)^2$. We first present an important lemma establishing independence of estimates corresponding to different windows.
 
\begin{lemma}[Independence of estimates]
Let $\y^{(i)} = \A^{(i)}\x^{(i)} + \w^{(i)}$, $i=0,1,\cdots,$ and $\{\w^{(i)}\}$ be independent, zero mean random vectors. The estimates $\{\hat{\x}^{(i)}\}$ obtained by LASSO,
$$\hat{\x}^{(i)}:=\argmin_{\x} ||\A\x - \y^{(i)}|| + \lambda ||\x||_1$$
are independent (conditioned on the input stream $\x := \{x_i\}$\footnote{This accounts for the general case of a random input source $\x$, where noise $\{\w^{(i)}$\} is independent of $\x$}).
\end{lemma}
\begin{proof}
The objective function of LASSO 
\begin{align*}
\left(\x,\w \right) \mapsto f(\x,\w) \coloneqq \Vert \A^{(i)}\x - \A^{(i)}\x^{(i)} - \w \Vert_2^2 + \lambda \Vert \x \Vert_1
\end{align*}
is jointly continuous in $(\x,\w)$, and the mapping obtained by minimizing over $\x$
\begin{align*}
\w \mapsto \min_{\x} f(\x,\w) \coloneqq g(\w)
\end{align*}
is continuous, hence Borel measurable. Thus, the definition of independence and the fact that $\w^{(i)}, \w^{(j)}$ are independent for $i\ne j$ concludes the proof.
\end{proof}
The expected $\ell_2$-norm of the reconstruction error satisfies:
\begin{align*}
\evc{x}{ (\bar{x}_i - x_i)^2}
&= \evc{x}{ \left(  \frac{1}{n} \sum_{j=i-n+1}^{i} \hat{x}^{(j)}_{i-j} - x_i   \right)^2 } \notag\\ 
&= \left( \evc{x} { \hat{x}^{(i)}_{0}} - x_i \right)^2 \notag + \frac{1}{n} \evc{x}{\left(\hat{x}^{(i)}_{0} - \evc{x}{\hat{x}^{(i)}_{0}}\right)^2},
\end{align*}
where we have used $\Cov{\hat{x}^{(j)}_{i-j}, \hat{x}^{(k)}_{i-k}}= 0$ for $j \neq k$, $j,k \in \{i-n+1,\dots,i\}$ which follows from independence. 
The resulting equality is the so called \emph{bias-variance} decomposition of the estimator. Note that as the window length is increased, the second term goes to zero and the reconstruction error asymptotically converges to the square of the  LASSO bias\footnote{LASSO estimator is \emph{biased} as a mapping from $\R^m\to\R^n$ with $m<n$.}. 

We have seen that averaging helps improve estimation accuracy. However, averaging, alone, is not enough for good performance, cf. Sec.~\ref{sec:simulations}, since the error variance is affected by the LASSO bias, even for large values of window size $n$. In the sequel, we propose a non-linear scheme for combining estimates from multiple windows which can overcome this limitation.

\subsection{The Proposed Algorithm}
\label{subsec:algo_RCS}

\begin{figure*}[t]
\vspace{8pt}
\centering
\includegraphics[width = \linewidth]{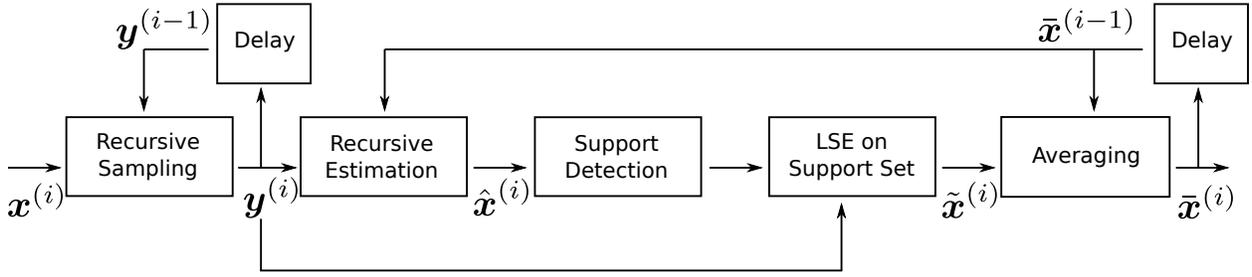}
\caption{Architecture of RCS.}
\label{fig:blockdiagram}
\vspace{-18pt}
\end{figure*}

In the previous section, we pointed out that leveraging the overlaps between windows--through averaging LASSO estimates--cannot yield an unbiased estimator, and the error variance does not go to 0 for large values of window size $n$. The limitation is the indeterminacy in the support of the signal--if the signal support is known, then applying least squares estimation (LSE) to an \emph{overdetermined} linear systems yields an unbiased estimator. In consequence, it is vital to address \emph{support detection}.

We propose a two-step estimation procedure for recovering the data stream: At first, we obtain the LASSO estimates $\{ \hat{\x}^{(i)} \}$ which are fed into a \emph{de-biasing} algorithm. For de-biasing, we estimate the signal support and then perform LSE on the support set in order to obtain estimates $\tilde{\x}^{(i)}$. The estimates obtained over successive windows are subsequently averaged. The block diagram of the method and the pseudocode for the algorithm can be seen in Figure \ref{fig:blockdiagram} and Algorithm \ref{alg:recursivecs}, respectively. In step 8, we show a recursive estimation of averages, applicable to an online implementation. In the next section, we present an efficient method for support detection with provable performance guarantees.

\begin{algorithm}[!ht]
\caption{\textbf{Recursive Compressed Sensing}}
\begin{algorithmic}[1]
\Require $\A^{(0)} \in \mathbb{R}^{m \times n}$, $\{ x_i \}_{i=0,1,\dots}, \lambda \geq 0$
\Ensure estimate $\{ \bar{x}_i \}_{i=0,1,\dots}$.
\State initialize signal estimate: $\{ \bar{x} \} \gets \{ 0 \}$
\For{$i = 0,1,2,\dots$}
\State $\x^{(i)} \gets \left[x_i \; x_{i+1} \; \dots \; x_{i+n-1}\right]^\T$ 
\State $\y^{(i)} \gets \A^{(i)} \x^{(i)} + \w^{(i)}$ 
\Comment{encoding}
\State $\hat{\x}^{(i)} \gets \underset{\vec{x} \in \R^n}{\argmin} \Vert \A^{(i)} \vec{x} - \y^{(i)} \Vert_2^2 + \lambda \Vert \vec{x} \Vert_1 $
\Comment{LASSO}
\State $I \gets \supp \left( \hat{\x}^{(i)} \right)$
\Comment{support  estimation}
\State $\tilde{\x}^{(i)} \gets  \underset{ {\substack{\vec{x} \in \R^n \\ \vec{x}_{\mathcal{I}^{c}} = \vec{0}}}}{\argmin} \Vert \A^{(i)} \vec{x} - \y^{(i)} \Vert_2^2 $
\Comment{LSE}
\State $\bar{x}_{i+j} \gets \left( (k_i(j)-1) \bar{x}_{i+j} + \tilde{x}_j^{(i)} \right) / k_i(j)$ for $j = 0,\dots,n-1$ where $k_i(j) = \min\{i+1,n-j\}$
\Comment{update average estimates}
\State $\A^{(i)} \gets \A^{(i-1)} \vec{\P} $
\Comment{for recursive sampling}
\EndFor
\end{algorithmic}
\label{alg:recursivecs}
\end{algorithm}

\subsection{Voting strategy for support detection}
\label{subsec:voting}


Recall the application of LASSO to signal support estimation covered in Section \ref{sec:background}. In this section, we introduce a method utilizing supports estimated over successive windows for robust support detection even in high measurement noise. At first step, LASSO is used for obtaining estimate $\hat{\x}^{(i)}$, which is then used as input to a voting algorithm for estimating the non-zero positions. Then, ordinary least squares are applied to the overdetermined system obtained by extracting the columns of the sampling matrix corresponding to the support. The benefit is that, since LSE is an unbiased estimator, averaging estimates obtained over successive windows may eliminate the bias, and so it is possible to converge to \emph{true} values as the window length increases. 


In detail, the two-step algorithm with voting entails solving LASSO:
$$\hat{\x}^{(i)} = \underset{\vec{x} \in \R^n}{\argmin}\: \left( \Vert \A^{(i)} \vec{x} - \y^{(i)}\Vert_2^2 + \lambda \Vert \vec{x} \Vert_1 \right),$$ 
then identifying the indices having magnitude larger than some predetermined constant $\xi_1 > 0$, in order to estimate the support of window  $\x^{(i)}$ by:
\begin{equation}\label{eq:votes_thresholding}
\hat{I}_{i} \coloneqq \left \{j: \vert \hat{\x}^{(i)}_j \vert \geq \xi_1 \right \}. 
\end{equation}
The entries of this set are given a \emph{vote}; the total number of votes determines whether a given entry is zero or not.  
Formally, we define the sequence containing the cumulative votes as $\{v_i\}$ and the number of times an index $i$ is used in LSE as $\{l_i \}$. At the beginning of the algorithm $\{v_i\}$ and $\{l_i \}$ are all set to zero. For each window, we add votes on the positions that are in the set $\hat{I}{i}$ as $v_{\hat{I}_{i}+i} \gets v_{\hat{I}_{i}+i} + 1$ (where the subscript $\hat{I}_{i}+i$ is used to translate the indices within the window to global indices on the streaming data). By applying threshold $\xi_2 \in \Z^+$ on the number of votes $\{v_i\}$, we get indices that have been voted sufficiently many times to be accepted as non-zeros and store them in: 
\begin{equation}\label{eq:support_detected}
R_{i} = \left \{j: v_{j+i} \geq \xi_2, \ j = 0,\dots,n-1 \right \}.
\end{equation} 
Note that the threshold $\xi_2 \in \{1,\cdots,n\}$ is equal to the \emph{delay} in obtaining estimates. This can be chosen such that $\vert R_{i} \vert < m$, hence yielding an \emph{overdetermined} system for the LSE.  
Subsequently, we solve the overdetermined least squares problem based on these indices in $R_i$,
\begin{align}
\tilde{\x}^{(i)} = \argmin_{\x \in \R^n, \x_{R_i^{c}}=0} \Vert \A^{(i)} \x - \y^{(i)} \Vert_2^2.
\end{align}
This problem can be solved in closed form, $\tilde{\x}_{R_i}^{(i)} = \left( \A_{R_i}^{(i)\T} \A^{(i)}_{R_i}\right)^{-1} \A^{(i)\T}_{R_i} \y^{(i)}$, where $\tilde{\x}_{R_i}^{(i)}$ is the vector obtained by extracting elements indexed by $R_i$, and $\A^{(i)}_{R_i}$ is the matrix obtained by extracting columns of $\A^{(i)}$ indexed by $R_i$. Subsequently, we increment the number of recoveries for the entries used in LSE procedure as $l_{R_i+i} \gets l_{R_i+i} + 1$, and the average estimates are updated based on the recursive formula  $\bar{x}_{i+j} \gets \frac{l_{i+j}-1}{l_{i+j}} \bar{x}_{i+j} +\frac{1}{\l_{i+j}} \tilde{x}_j$, for $j \in R_i$.

\section{Extensions}
\label{sec:discussion}

In this section we present various extensions to the algorithm.

\subsection{Sliding window with step size $\tau$}
\label{subsec:swwvo}

Consider a generalization in which sensing is performed via recurring windowing with a step size $0 < \tau \leq n$ , i.e., $\vec{x}^{(i)} := \begin{bmatrix}x_{i\tau} & x_{i\tau +1} & \dots &x_{i\tau+n-1} \end{bmatrix}^{\T}$. 

We let $\eta_i$ denote the \emph{sampling efficiency}, that is the ratio of the total number of samples taken until time $n+i$ to the number of retrieved entries, $n+i$. For one window, sampling efficiency is $m/n$. By the end of $i^{th}$ window, we have recovered $n+(i-1)\tau$ elements while having sensed $im$ many samples. The asymptotic sampling efficiency is:
\begin{align*}
\eta := \lim_{i \to \infty} \eta_i & = \lim_{i \to \infty} \frac{im}{n+(i-1)\tau} = \frac{m}{\tau}.
\end{align*} 
The alternative is to encode using a rank-$\tau$ update (i.e., by recursively sampling using the matrix obtained by circularly shifting the sensing matrix $\tau$ times, $\A^{(i+1)} = \A^{(i)} \vec{P}^{\tau}$). In this scheme, for each window we need to store $\tau$ scalar parameters; for instance, this can be accomplished by a least-squares fit of the difference $\vec{y}^{(i+1)} - \vec{y}^{(i)}$ in the linear span of the first $\tau$ columns of  
$\vec{A}^{(i)}$ (cf. \eqref{eq:rank1}). The asymptotic sampling efficiency becomes\footnote{Note that when $\tau \geq m$, 
recording samples $\{\y^{(i)}\}$ directly (as opposed to storing $\tau$ parameters) yields better efficiency $\eta = \frac{m}{\tau} \le 1$.}:
\begin{align*}
\eta = \lim_{i \to \infty} \frac{m+(i-1)\tau}{n+(i-1)\tau} = 1.
\end{align*}

In the latter case, the recursive sampling approach is asymptotically equivalent to taking one sample for each time instance. Note, however, that the benefit of such an approach lies in noise suppression. By taking overlapping windows each element is sensed at minimum $\floor{n/\tau}$ many times, hence collaborative decoding using multiple estimates can be used to increase estimation accuracy.

\subsection{Alternative support detection}\label{sec:alt_supp}

The algorithm explained in Sec.~\ref{subsec:voting} selects indices to be voted by thresholding the LASSO estimate as in \eqref{eq:votes_thresholding}. An alternative approach is by leveraging the estimates obtained so far: since we have prior knowledge about the signal at $i^{th}$ window $\tilde{\x}_{[0]}^{(i)}$ from $(i-1)^{th}$ window, $\tilde{\x}^{(i-1)}$, we can annihilate the sampled signal as:
\begin{align*}
\tilde{\y}^{(i)} := \y^{(i)} - \A^{(i)}\tilde{\x}_{[0]}^{(i)}.
\end{align*}
If the recovery of the previous window was perfect, $\tilde{\y}^{(i)}$ would be equal to $\av^{(i)}_n \x_{i+n-1} + \w^{(i)}$ and thus $\x_{i+n-1}$ can be estimated by LSE as $\x_{i+n-1} = \av^{(i)\T}_n \tilde{\y}^{(i)}$. However, since the previous window will have estimation errors, this does not hold. In such case, we can again use LASSO to find the estimator for the error between the true signal and estimate as:
\begin{align*}
	\hat{\x}^{(i)} = \underset{\x}{\text{arg min}}\: \left( \Vert A^{(i)} \x - \tilde{\y}^{(i)}\Vert_2^2 + \lambda \Vert \x \Vert_1 \right)
\end{align*}
and place votes on the $\xi_3 \in \Z^+$ indices of highest magnitudes, i.e., 
\begin{align}
S_t^{(i)} = \left \{j: \hat{|\x}^{(i)}_j| \geq z_{\xi_3} \right\}.
\label{eq:votes_fixedmany}
\end{align}
instead of \eqref{eq:votes_thresholding}. The rest of the estimation method remains the same. Since the noise is i.i.d., the expected number of votes a non-support position collects is less than $\xi_3$. Thus the threshold $\xi_2$ in \eqref{eq:support_detected} needs to satisfy $\xi_3 \leq \xi_2 \leq n$ in order to eliminate false positives.

Last, in the spirit of \emph{recursive least squares} (RLS)~\cite{Kumar}, we consider joint identification over multiple windows with \emph{exponential forgetting}. Let $T$ be the \emph{horizon}, i.e., the number of past windows considered in the estimation of the current one. Also, let $\rho \in [0,1)$. For the $i-$th window\footnote{We consider the case $i\ge T$, and $\tau=1$ for notational simplicity.} we solve:
\begin{equation}
\x^* := \argmin_{\x \in \R^{n+T}} \sum_{j=i-T}^i \rho^{j-i}\left( ||\A^{(j)}\x^{(j)} - \y^{(j)}||_2^2 + \lambda||\x^{(j)}||_1]\right),
\end{equation}
where the decision vector $\x$ corresponds to $[x_{i-T},\cdots,x_{i+n-1}]$, and we set  $\hat{\x}^{(i)}:=[\x^*_{T},\cdots,\x^*_{n+T}$. It is interesting to point out that this optimization problem can be put into standard LASSO form by weighting the entries of the decision vector at a pre-processing step, so standard numerical schemes can be applied. Note that the computational complexity is increasing with $T$ and, unlike traditional RLS, $T$ has to be finite. 

\subsection{Expected Signal Sparsity}

We have considered, so far, the case that each window $\x^{(i)}$ is $\kappa$-sparse. However, the most general case is when the data stream is $\kappa$-sparse \emph{on average}, in the sense that: 
$$\bar{\kappa}:= \limsup_{N} \frac{1}{N} \sum_{i=0}^N 1_{x_i \ne 0} \le \kappa.$$ 
In such case, one can simply design RCS based on 
some value $\kappa \ge \bar{\kappa}$, and leverage  Theorem~\ref{thm:LASSO_error} to incorporate the error due to \emph{model-mismatch} in the analysis (cf. Theorem \ref{main_thm}).  
For both analysis and experiments we adopt a random model, in which each entry of the data stream is generated i.i.d. according to: 
\begin{align}
f_{X_i}(x) =
\begin{cases}
(1-p)\delta(x) + \frac{1}{2p} &\mbox{if } x \in [-A,A] \\
0 & \mbox{o.w. }
\end{cases}
\label{eq:signal_model}
\end{align}
where $p \in \left(0,1\right]$. This is the density function of a random variable that is $0$ with probability $1-p$ and sampled uniformly over the interval $[-A,A]$ otherwise\footnote{Note that the case  $p=0$ is trivially excluded since then the data stream is an all-zero sequence. }. The average sparsity of the stream is $\bar{\kappa} = p$. 

We can calculate the mean error due to model-mismatch by: 
\begin{small}
\begin{align*}
\ev{\Vert \X-\X_\kappa \Vert_1} 
&=  A \sum_{k=\kappa+1}^{n} {n \choose k} p^k (1-p)^{n-k} \sum_{i=\kappa+1}^{k} \left( 1 - \frac{i}{k+1} \right),
\end{align*}
\end{small}
where $\X$ denotes an $n$-dimensional random vector with entries generated i.i.d. from \eqref{eq:signal_model}, and $\X_k$ is obtained from $\X$ by setting all but its $\kappa$ largest entries equal to zero.  

The result is a function of the window length, $n$, the sparsity $\kappa$ used in designing sensing matrices (e.g., we can take mean sparsity $\kappa = pn$), and the probability of an element being nonzero, $p$. In place of the (rather lengthy, yet elementary) algebraic calculations  we illustrate error due to model-mismatch in Fig.~\ref{fig:expected_norm_plots}. We point out that we can analytically establish boundedness for all values of $p,n$, so our analysis in Sec.~\ref{sec:analysis} carries over unaltered. The analysis of other distributions on the magnitudes of non-zero entries can be carried out in a similar way.


\begin{figure*}[!t]
        \centering
        \begin{subfigure}[h]{0.48\linewidth}
                \centering
                \includegraphics[width = \linewidth]{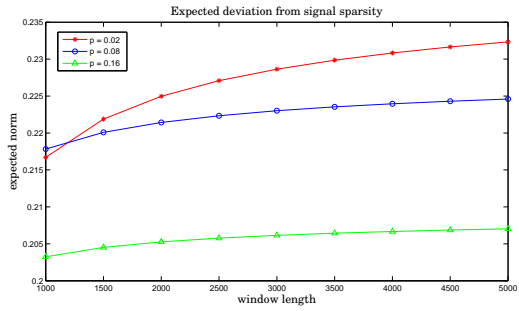}
                \caption{$\ev{\Vert \x-\x_\kappa \Vert_1}$ vs. window length $n$;$\kappa = np$.}
                \label{fig:norm_expected1}
        \end{subfigure}
        \quad
        \begin{subfigure}[h]{0.48\linewidth}
                \centering
                \includegraphics[width = \linewidth]{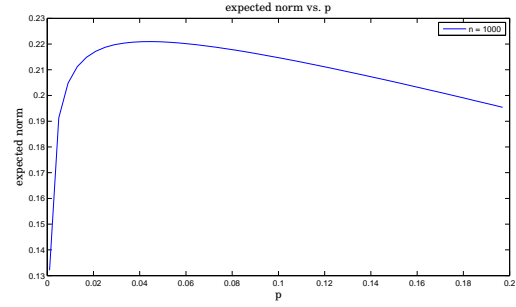}
                \caption{$\ev{\Vert \x-\x_\kappa \Vert_1}$ vs. probability of non-zero $p$; $\kappa = np$, $n = 1000$.}
                \label{fig:norm_expected2}
        \end{subfigure}
        \caption{Error due to model-mismatch: Expected deviation from signal sparsity $\ev{\Vert \x-\x_\kappa \Vert_1}$.}
        \label{fig:expected_norm_plots}
\vspace{-10pt}
\end{figure*}

\section{Analysis}
\label{sec:analysis}

In this section we analyze the estimation error variance and computational complexity of the proposed method.

\subsection{Estimation Error Variance}

Given $\{x_i\}$ we give a bound on the normalized error variance of each window defined as:
\begin{align*}
\text{NE}(i) \coloneqq \ev{\frac{\Vert \bar{\x}^{(i)} - \x^{(i)} \Vert_2}{\Vert \x^{(i)} \Vert_2}}.
\end{align*}
Note that for an ergodic data source this index (its inverse expressed in log-scale) corresponds to average Signal to Residual Ratio (SRR).
\begin{theorem}[Normalized Error of RCS]\label{main_thm}

Under the assumptions of Theorem \ref{theorem:modelselection} and given $\A^{(0)}$ satisfying RIP with $\delta_{\kappa}$, for $\{x_i\}_{i=0,1,\dots}$ satisfying $\Vert \x^{(i)} \Vert_0 \geq \Omega\left(\kappa \right)$, NE(i) satisfies:
\begin{align*}
\text{NE(i)} \leq P^n \cdot c_1 \frac{1}{\sqrt{n \log n}} + (1-P^n)) \left(c_2 + c_3 \frac{\sqrt{m}}{\sqrt{\kappa \log n}} \right), 
\end{align*}
where $c_1$, $c_2$ and $c_3$ are constants, and
\begin{align*}
P^n \geq \left( 1 - \frac{2}{n\sqrt{2\pi \log n}} - \frac{2\kappa}{n^2}-O\left( \frac{1}{n^{2 \log 2}} \right) \right)^{2n-1}
\end{align*}
\end{theorem}
\vspace{2mm}
\begin{proof}
Defining the event $\text{S}_{2n-1}$ $\coloneqq$ \{support is detected correctly on $2n-1$ consecutive windows\}\footnote{Note that the definition of the ``success'' set is very conservative.},
we have the following equality for NE given $\{x_i\}_{i=0,1,\dots}$:
\begin{align*}
\text{NE}(i) = P(\text{S}_{2n-1}) \cdot \evc{x,S_{2n-1}}{\frac{\Vert \bar{\x}^{(i)} - \x^{(i)} \Vert_2}{\Vert \x^{(i)} \Vert_2}}
+ (1-P(\text{S}_{2n-1})) \cdot \evc{x,S_{2n-1}^c}{\frac{\Vert \bar{\x}^{(i)} - \x^{(i)} \Vert_2}{\Vert \x^{(i)} \Vert_2}},
\end{align*}
where, dropping the subscript $2n-1$, and using S as a shorthand notation for $\text{S}_{2n-1}$, we have:
\begin{align*}
P(\text{S}) \geq \left( 1 - \frac{2}{n\sqrt{2\pi \log n}} - \frac{2\kappa}{n^2}-O\left( \frac{1}{n^{2 \log 2}} \right) \right)^{2n-1},
\end{align*}
by Theorem \ref{theorem:modelselection}.

In S, by LSE we get:
\begin{align*}
\evc{x,S}{\frac{\Vert \bar{\x}^{(i)} - \x^{(i)} \Vert_2}{\Vert \x^{(i)} \Vert_2}}
&\leq \frac{1}{\Vert \x^{(i)} \Vert_2} \sqrt{ \evc{x,S}{\Vert \bar{\x}^{(i)} - \x^{(i)} \Vert^2_2} } \\
&\overset{(a)}\leq \frac{\sigma \sqrt{\kappa}}{\Vert \x^{(i)} \Vert_2 \sqrt{n(1-\delta_{\kappa})}},
\end{align*}
where $(a)$ follows from
\begin{align*}
\evc{x,S}{\Vert \bar{\x}^{(i)} - \x^{(i)} \Vert_2^2} &=\evc{x,S}{\sum_{j \in I} \left( \bar{x}_{i+j} - x_{i+j} \right)^2} \\
&=\evc{x,S}{\sum_{j \in I} \left( \frac{1}{n} \sum_{t=0}^{n-1} \tilde{x}^{(i+j-t)}_{j+t} - x_{i+j} \right)^2} \\
&= \sum_{j \in I} \sum_{t,r=0}^{n-1} \evc{x,S}{ \frac{ \left( \tilde{x}^{(i+j-t)}_{j+t} - x_{i+j} \right)\left(   \tilde{x}^{(i+j-r)}_{j+r} - x_{i+j} \right)}{n^2}} \\
&= \frac{1}{n^2} \sum_{j \in I} \sum_{t=0}^{n-1}
\evc{x,S}{ \left(   \tilde{x}^{(i+j-t)}_{j+t} - x_{i+j} \right)^2} \\
& \overset{(b)}\leq \frac{1}{n^2} \sum_{j \in I} \sum_{t=0}^{n-1}
\frac{\sigma^2}{1-\delta_{\kappa}} \leq 
\frac{\kappa \sigma^2}{n (1-\delta_{\kappa})}
\end{align*}
where $I = \supp (\hat{\x}^{(i)} )$, is also equal to $\supp (\x^{(i)})$ given S, and $(b)$ follows since the covariance matrix of LSE is $\sigma^2 (\A^{T}_I \A^{ }_I)^{-1}$ and by RIP we have all of the eigenvalues of $\A^{\T}_{I} \A^{ }_{I}$ greater than $(1-\delta_{\kappa})$ since $(1-\delta_{\kappa})\Vert \x \Vert_2^2 \leq \Vert \A \x \Vert_2^2$ for all $\x$ $\kappa$-sparse.
To bound the estimation error in $\text{S}^c$, note that independent of the selected support, by triangle inequality, we have:
\begin{align*}
\Vert \A^{(i)} \tilde{\x}^{(i)} - \y^{(i)} \Vert_2  
&\overset{(a)} \leq \Vert \y^{(i)} \Vert_2 \\ 
& \leq \Vert \A^{(i)}\x^{(i)} \Vert_2 + \Vert \w^{(i)} \Vert_2 \\ 
&\leq (1 + \delta_{\kappa}) \Vert \x^{(i)} \Vert_2 + \Vert \w^{(i)} \Vert_2,
\end{align*}
and
\begin{align*}
\Vert \y^{(i)} \Vert_2 \geq \Vert \A^{(i)} \tilde{\x}^{(i)} - \y^{(i)} \Vert_2 & \geq \Vert \A^{(i)} \tilde{\x}^{(i)}\Vert_2 - \Vert \y^{(i)} \Vert_2 \\ 
&\geq (1 - \delta_{\kappa}) \Vert \tilde{\x}^{(i)} \Vert_2 - \Vert \y^{(i)} \Vert_2,
\end{align*}
where $(a)$ follows since
\begin{equation*}
\begin{aligned}
\tilde{\x}^{(i)} = \argmin_{\x \in \R^n, \x_{I^c}=0} \Vert \A^{(i)}\x^{(i)} - \y^{(i)} \Vert_2.
\end{aligned}
\end{equation*}
From these two inequalities we have:
\begin{align*}
\Vert \tilde{\x}^{(i)}\Vert_2 &\leq \frac{2}{1-\delta_{\kappa}} \Vert \y^{(i)}\Vert_2 \\
& \leq \frac{2}{1 - \delta_{\kappa}} \left( (1+\delta_{\kappa}) \Vert \x^{(i)} \Vert_2 + \Vert \w^{(i)} \Vert_2 \right).
\end{align*}
By applying triangle inequality once more we get:
\begin{align*}
\Vert \tilde{\x}^{(i)} - \x^{(i)} \Vert_2 & \leq \Vert \tilde{\x}^{(i)} \Vert_2 + \Vert \x^{(i)} \Vert_2 \\
& \leq \Vert \x^{(i)} \Vert_2 \left(1 + \frac{2(1+\delta_{\kappa})}{1-\delta_{\kappa}} \right) + \frac{2\Vert \w^{(i)} \Vert_2}{1-\delta_{\kappa}}.
\end{align*}

Thus in $\text{S}^c$ we have:
\begin{align*}
\evc{x,\text{S}^c}{\frac{\Vert \bar{\x}^{(i)} - \x^{(i)} \Vert_2}{\Vert \x^{(i)} \Vert_2}} 
&= \frac{1}{\Vert \x^{(i)} \Vert_2} \evc{x,\text{S}^c}{\Vert \bar{\x}^{(i)} - \x^{(i)} \Vert_2} \\
& \overset{(b)}\leq \frac{1}{\Vert \x^{(i)} \Vert_2} \evc{x,\text{S}^c}{\Vert \tilde{\x}^{(i)} - \x^{(i)} \Vert_2} \\
& \leq \left(1 + \frac{2(1+\delta_{\kappa})}{1-\delta_{\kappa}} \right) + \frac{2}{1-\delta_{\kappa}} \frac{\ev{\Vert \w^{(i)} \Vert_2}}{\Vert \x^{(i)} \Vert_2},
\end{align*}
where $(b)$ follows from Jensen's inequality. 

We get the result by taking the expectation over $\{ x_i\}_{i=0,1,\dots}$ and noting by the assumptions of Theorem \ref{theorem:modelselection} we have $\ev{\Vert \w^{(i)} \Vert_2} \leq \sigma\sqrt{m}$, $\vert x_{i+j} \vert \geq 8\sigma \sqrt{2\log n}$ where $j \in \supp \left( \x^{(i)} \right) $ and $\Vert  \x^{(i)}\Vert_0 \geq \Omega \left( \kappa \right)$.
\end{proof}

\begin{corollary}
For \emph{sublinear} 
sparsity $\kappa = O(n^{1-\epsilon})$,  non-zero data entries with magnitude $\Omega(\sqrt{logn})$, and obtained samples  $m = O(\kappa \log n)$, where $n$ is the window length, the normalized error goes to 0 as $n \rightarrow \infty$.
\end{corollary}

\begin{proof}
For $\kappa = O(n^{1-\epsilon})$ we have $P^n \geq \left( 1 - O\left(\frac{1}{n\sqrt{\log n}}\right) \right)^{2n-1}$, and from the assumptions we have $c_3 \frac{\sqrt{m}}{\sqrt{\kappa \log n}}$ constant. We get the result by noting $P^n$ goes to 1 as the window length $n$ goes to infinity.
\end{proof}

\begin{rem}[Error of voting]
Note that the exact same analysis applies directly to voting by invoking \emph{stochastic dominance}: for any positive threshold $\xi_1\le 8\sigma\sqrt{\log n}$, correct detection occurs in a superset of $S_{2n-1}$ (defined by requiring perfect detection in all windows, i.e., $\xi_1=0,\xi_2=n$). 
\end{rem}

\begin{rem}[Dynamic range\footnote{The  authors would like to thank Pr. Yoram Bresler for a fruitful comment on dynamic range.}]\label{rem_dynamic} Note that in a real scenario, it may be implausible to increase the window length arbitrarily, because the \emph{dynamic range} condition $\underset{i\in I_{\x}}{\min} \; \vert x_i \vert > 8 \sigma \sqrt{2 \log n}$ may be violated. This observation may serve to provide a means for selecting $n$ (the good news being that the lower bound increases very slowly in window length, only as $\sqrt{\log n}$). In multiple simulations we have observed that this limitation is actually negligible: $n$ can be selected way beyond this barrier without any compromise in increasing estimation accuracy. 
\end{rem}
Last, we note that it is possible to carry out the exact same analysis for general step size $\tau$; as expected, the upper bound on normalized error variance is increasing in $\tau$,  but we skip the details for length considerations.

\subsection{Computational Complexity Analysis}\label{sec:comp_complexity}

In this section, we analyze the computational complexity of RCS. Let $i$ be the window index, $\A^{(i)} \in \R^{m \times n}$ be the sampling matrix, and recall the extension on $\tau$, the number of shifts between successive windows. By the end of $i^{th}$ window, we have recovered $n + (i-1)\tau$ many entries. As discussed in Section~\ref{sec:recursivecs}, the first window is sampled by $\A^{(0)}\x^{(0)}$; this requires $O(mn)$ basic operations (additions and multiplications). After the initial window, sampling of the window $\x^{(i)} = \begin{bmatrix} x_{i\tau} & x_{i\tau+1} & \cdots & x_{i\tau+n-1} \end{bmatrix}^{\T}$ is achieved by recursive sampling having rank-$\tau$ update with complexity $O(m\tau)$. Thus, by the end of $i^{th}$ window, total complexity of sampling is $O(mn) + O(m\tau)i$. The \emph{encoding complexity} is defined as the normalized complexity due to sampling over the number of \emph{retrieved entries}:
\begin{equation}\label{encoding_complexity}
C_e := \lim_{i\to +\infty} \frac{C_e(i)}{n+(i-1)\tau},
\end{equation}
where $C_e(i)$ denotes the total complexity of encoding all stream entries $0,1,\hdots,i$. 
For recursive sampling $C_e = O(m)$ while for non-recursive we have $C_e = O(mn /\tau)$; note that by recursively sampling the input stream, the complexity is reduced by $\frac{n}{\tau}$.

The other contribution to computational complexity is due to the iterative solver, where the expected complexity can be calculated as the number of operations of a single iteration multiplied by the expected number of iterations for convergence. The latter is a function of the distance of the starting point to the optimal solution \cite{Bertsekas1995}, which we bound in the case of using \emph{recursive estimation}, as follows:

\begin{lemma}
Using $\hat{\x}^{(i)}_{[0]} = 
\begin{bmatrix}
x_{*\tau}^{(i-1)} & \dots & x_{*n-1}^{(i-1)} & \vec{0}^\T_\tau
\end{bmatrix}^\T
$ as the starting point we have:
\begin{align*}
\Vert \hat{\x}_{[0]}^{(i)} - \x^{(i)}_{*} \Vert_2  &\leq  c_{0} \Vert \x^{(i-1)} - \x^{(i-1)}_\kappa \Vert_1  / \sqrt{\kappa} \notag \\
&\quad+ c_{0} \Vert \x^{(i)} - \x^{(i)}_\kappa \Vert_1  / \sqrt{\kappa} \notag \\
&\quad+ c_{1} \tilde{\sigma} + \Vert \begin{bmatrix}
x^{(i)}_{n-\tau} \dots x^{(i)}_{n-1}
\end{bmatrix} \Vert_2,
\end{align*}
where $c_0$ and $c_1$ are constants.
\end{lemma}
\begin{proof}
Defining:
\begin{align*}
&{\vec{e}'}^{(i)} \coloneqq  
\begin{bmatrix}
x_{*\tau}^{(i-1)} \dots x_{*n-1}^{(i-1)} & \vec{0}_{\tau}^\T
\end{bmatrix}^\T -
\begin{bmatrix}
x_{*0}^{(i)} \dots x_{*n-1}^{(i)}
\end{bmatrix}^\T \\
&{\vec{e}}^{(i)} \coloneqq  
\x_*^{(i)} - \x^{(i)},
\end{align*}
we have
\begin{align*}
{\vec{e}'}^{(i)} = 
\begin{bmatrix}
x_{*\tau}^{(i-1)} \dots x_{*n-1}^{(i-1)} \ \ \vec{0}_{\tau}^\T
\end{bmatrix}^\T - \x^{(i)} 
+ \x^{(i)} - \vec{x}_*^{(i)}.
\end{align*}
Taking the norm and using triangle inequality yields:
\begin{align*}
\Vert {\vec{e}'}^{(i)} \Vert_2 &\leq \Vert \vec{e}^{(i-1)} \Vert_2 + \Vert \vec{e}^{(i)} \Vert_2 + \Vert \begin{bmatrix}
x^{(i)}_{n-\tau} \dots x^{(i)}_{n-1}
\end{bmatrix} \Vert_2.
\end{align*}
Using Theorem~\ref{th:candes} we get:
\begin{align}
\Vert {\vec{e}'}^{(i)} \Vert_2 &\leq  c_{0} \Vert \x^{(i-1)} - \x^{(i-1)}_\kappa \Vert_1  / \sqrt{\kappa} \notag \\
&+ c_{0} \Vert \x^{(i)} - \x^{(i)}_\kappa \Vert_1  / \sqrt{\kappa} \notag \\
&+ c_{1} \tilde{\sigma} + \Vert \begin{bmatrix}
x^{(i)}_{n-\tau} \dots x^{(i)}_{n-1}
\end{bmatrix} \Vert_2.
\label{eq:noisycomp}
\end{align}
\end{proof}

Exact computational complexity of each iteration depends on the algorithm. Minimally, iterative solver for LASSO requires multiplication of sampling matrix and the estimate at each iteration which requires $O(mn)$ operations. In an algorithm where cost function decays sublinearly (e.g., $1/t^2$), as in FISTA, the number of iterations, t, required for obtaining $\hat{\x}_{[t]}$ such that $G(\hat{\x}_{[t]}) - G(\x_*) \leq \epsilon$, where $\x_*$ is the optimal solution, is proportional to $\Vert \x_{[0]} - \x_* \Vert_2$ (e.g., $\Vert \x_{[0]} - \x_* \Vert_2/\sqrt{\epsilon}$) where $\x_{[0]}$ is the starting point of the algorithm \cite{Beck2009}. From this bound, it is seen that average number of iterations is proportional to the Euclidean distance of the starting point of the algorithm from the optimal point. 

\begin{lemma}[Expected number of iterations]\footnote{We note in passing that this bound on the expected number of iterations is actually conservative, and can be improved based on a homotopy analysis of warm-start~\cite{recursive_LASSO,romberg}; this is beyond the scope of the current paper. 
}
For the sequence $\{x_i\}_{i=0,1,\dots}$ where $\Vert \x^{(i)} \Vert_0 \leq \kappa$ with the positions of non-zeros chosen uniformly at random and $\underset{j=0,\dots,n-1}{\max} \vert \x^{(i)}_j \vert = O \left( \sqrt{\log n} \right)$ for all $i$, the expected number of iterations for convergence of algorithms where cost function decays as $1/t^2$ is $O(\sqrt{(\kappa \tau \log n) / n})$ for noiseless measurements and $O(\sqrt{m})$ for i.i.d. measurement noise.
\end{lemma}
\begin{proof}
Since $\x^{(i)}$ is $\kappa$-sparse, the terms $ \Vert \x^{(i-1)} - \x^{(i-1)}_\kappa \Vert_1$ and $\Vert \x^{(i)} - \x^{(i)}_\kappa \Vert_1$ vanish in \eqref{eq:noisycomp}. By $\vert x_i \vert = O \left( \sqrt{\log n}\right)$ and uniform distribution of non-zero elements we have $\ev{\Vert \begin{bmatrix}
x^{(i)}_{n-\tau} \dots x^{(i)}_{n-1}
\end{bmatrix}\Vert_2} \leq \sqrt{(\kappa \tau \log n) / n}$.

With noisy measurements, the term $c_1 \tilde{\sigma}$ is related to the noise level. Since noise has distribution $\w^{(\cdot)} \sim \mathcal{N} \left(0,\sigma^2\vec{I} \right)$, the squared norm of the noise $\Vert \w^{(i)} \Vert_2^2$ has chi-squared distribution with mean $\sigma^2 m$ and standard deviation $\sigma^2 \sqrt{2m}$; probability of the squared norm exceeding its mean plus 2 standard deviations is small, hence we can pick $\tilde{\sigma}^2 = \sigma^2 \left(m + 2 \sqrt{2m} \right)$ \cite{Candes2006b} to satisfy the conditions of Theorem~\ref{th:candes}. Using this result in \eqref{eq:noisycomp}, we get $O(\sqrt{(\kappa \log n \tau) / n}) + O( \sqrt{m})$, where the second term dominates since $\tau \leq n$ not to leave out any element of the signal and $m \sim O(\kappa \log n)$. Hence it is found that the expected number of iterations is $O(\sqrt{m})$ in the noisy case. 
\end{proof}

The other source of complexity is the LSE in each iteration, which requires solving a linear $\kappa \times \kappa$ system that needs $O(\kappa^3)$ operations. Finally, averaging can be performed using $O(n/\tau)$ operations for each given entry. 
We define the \emph{decoding complexity} as the normalized complexity due to estimation over the number of \emph{retrieved entries}:
\begin{equation}\label{decoding_complexity}
C_d := \lim_{i\to +\infty} \frac{C_d(i)}{n+(i-1)\tau},
\end{equation}
where $C_d(i)$ denotes the total complexity of decoding all stream entries $0,1,\hdots,i$.
It follows that decoding complexity is equal to 
$C_d = O(\frac{m^{3/2}n + \kappa^3}{\tau})$,
using recursive estimation. 
To conclude, the asymptotic total complexity (per retrieved stream entry),
$$C=C_e + C_d,$$
is dominated by LASSO and LSE (based on the facts that $m \ge 1, \frac{n}{\tau}\ge 1$), therefore: 
\begin{equation}\label{decoding_complexity}
C = O(\frac{m^{3/2}n + \kappa^3}{\tau}),
\end{equation}
In Table~\ref{table_complexity} we demonstrate the total complexity for various sparsity classes $\kappa$, based on the fundamental relation $m =O(\kappa\log{\frac{n}{\kappa}})$~\cite{baraniuk2008}.  Note that the computational complexity is decreasing in $\tau$, while error variance is increasing in $\tau$. This trade-off can be used for selecting window length $n$ and step size $\tau$ based on desired estimation accuracy and real-time considerations.



\begin{table}[!t]
\centering
\caption{Computational complexity per entry as function of window length $n$ and step size $\tau$ for different sparsity classes.} \label{table_complexity} 
\begin{IEEEeqnarraybox}
[\IEEEeqnarraystrutmode\IEEEeqnarraystrutsizeadd{2pt}{1pt}]{v/c/v/c/v} \IEEEeqnarrayrulerow\\
& \kappa && \text{Computational Complexity}&\\
\IEEEeqnarraydblrulerow\\
\IEEEeqnarrayseprow[3pt]\\
& O(1) && O\left(n(\log n)^{3/2}/\tau \right) &\IEEEeqnarraystrutsize{0pt}{0pt}\\
 \IEEEeqnarrayseprow[3pt]\\
\IEEEeqnarrayrulerow\\
\IEEEeqnarrayseprow[3pt]\\
& O(\log n )  && O\left(n\left(\log n \cdot \log(n /\log n) \right)^{3/2} /\tau\right)
&\IEEEeqnarraystrutsize{0pt}{0pt}\\
\IEEEeqnarrayseprow[3pt]\\
\IEEEeqnarrayrulerow\\
\IEEEeqnarrayseprow[3pt]\\
& O(\sqrt{n}) && O(n^{\frac{3}{2}}/\tau) 
&\IEEEeqnarraystrutsize{0pt}{0pt}\\
\IEEEeqnarrayseprow[3pt]\\
\IEEEeqnarrayrulerow\\
\IEEEeqnarrayseprow[3pt]\\
& O(n) && O( n^3/\tau)  &\IEEEeqnarraystrutsize{0pt}{0pt}\\
\IEEEeqnarrayseprow[3pt]\\
\IEEEeqnarrayrulerow
\end{IEEEeqnarraybox}
\vspace{-10pt}
\end{table}

\section{Simulation Results}
\label{sec:simulations}

The data used in the simulations are generated from the random model \eqref{eq:signal_model}
\footnote{We also tested the case where the values of non-zero entries are generated i.i.d. from a Gaussian distribution; even though this model may violate the dynamic range assumption, cf. Rem. \ref{rem_dynamic}, the results are very similar.} with $p = 0.05$. The measurement model is $ \y^{(i)} = \A^{(i)} \x^{(i)}+ \vec{w}^{(i)}$ with $\vec{w}^{(i)} \sim \mathcal{N}\left( 0, \sigma^2 \vec{I} \right)$ where $\sigma \in \R^+$, and the sampling matrix is $\A^{(0)} \in \R^{m \times n}$ where $m = 6pn$ and $n$ is equal to the window length.

In the sequel, we test RCS as described in sections ~\ref{subsec:algo_RCS}, \ref{subsec:voting}. We have also experimented extensively on the extensions presented in Sec.~\ref{sec:alt_supp}, but do not present the results here because: a) the exponential-forgetting approach, alone, does not improve estimation accuracy while it incurs computation overhead, and b) the performance and run-time of generalized voting is no different than that of standard voting. 
  
\subsection{Runtime}

We experimentally test the speed gain achieved by RCS by comparing the average time required to estimate a given window while using FISTA for solving LASSO. RCS is compared against so called `na\"{i}ve approach', where the sampling is done by matrix multiplication in each window and FISTA is started from all zero vector. The average time required to recover one window in each case is shown in Figure~\ref{fig:runtime}.

\begin{figure}[h!]
\centering
\includegraphics[height = 0.35\textheight]{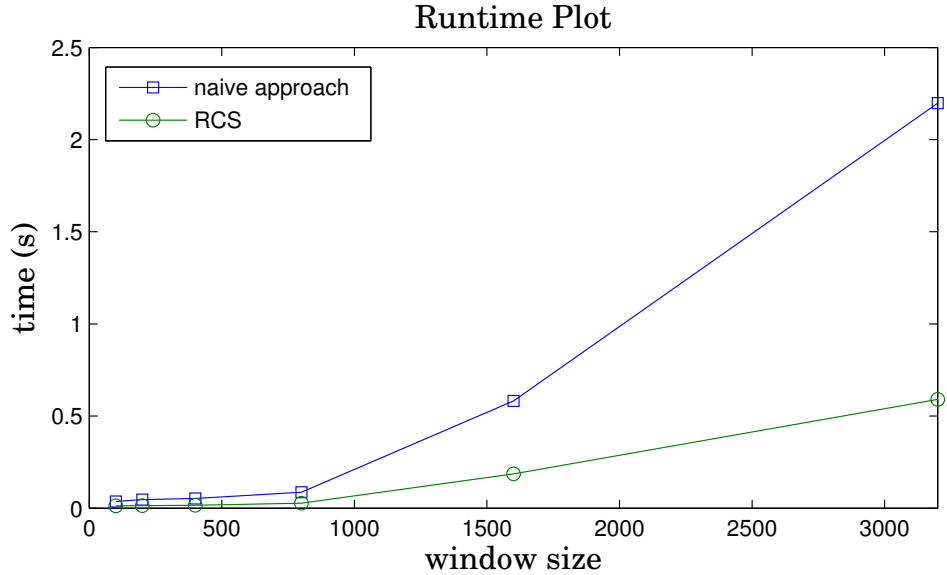}
\caption{Average processing time of RCS vs. traditional (non-recursive) CS over a single time window.}
\label{fig:runtime}
\vspace{-16pt}
\end{figure}

\subsection{Support Estimation}

We present the results of experiments on the support estimation using LASSO. In the measurements $\x \in \R^{6000}$, $\Vert \x \Vert_0 = 60$, $\A \in \R^{m \times 6000}$ is generated by i.i.d. Gaussian distribution with $A_{i,j} \sim \mathcal{N}(0, 1/m)$, and $\w$ has $\sigma = 0.1$. As suggested in Theorem \ref{theorem:modelselection} for these parameters, LASSO is solved with $\lambda =  4 \sigma \sqrt{2 \log n}$, and the nonzero entries of $\x$ are chosen so that $\underset{i=1,2,\dots,n}{\min} \vert x_i \vert \geq 3.34$ by sampling from $\mathcal{U}\left( \left[ -4.34,-3.34\right] \cup \left[ 3.34, 4.34\right] \right)$. In simulations, we vary the number of samples taken from the signal, $m$, and study the accuracy of support estimation by using 
\begin{align*}
\text{true positive rate} &= \frac{\vert\text{detected support} \cap \text{true support} \vert}{\vert \text{true support} \vert}\\
\text{false positive rate} &= \frac{\vert \text{detected support} \backslash \text{true support} \vert}{n - \vert \text{true support} \vert},
\end{align*}
where $\vert \cdot \vert$ denotes the cardinality of a set and $\backslash$ is the set difference operator.

The support is detected by taking the positions where the magnitude of the LASSO estimate is greater than threshold $\xi_1$ for values $0.01$, $0.1$, $1$. Figure \ref{fig:drfp} shows the resulting curves, obtained by randomly generating the input signal 20 times for each $m$ and averaging the results. It can be seen that the false positive rate can be reduced significantly by properly adjusting the threshold on the resulting LASSO estimates.

\begin{figure}[h!]
    	\centering
        \includegraphics[height = 0.35\textheight]{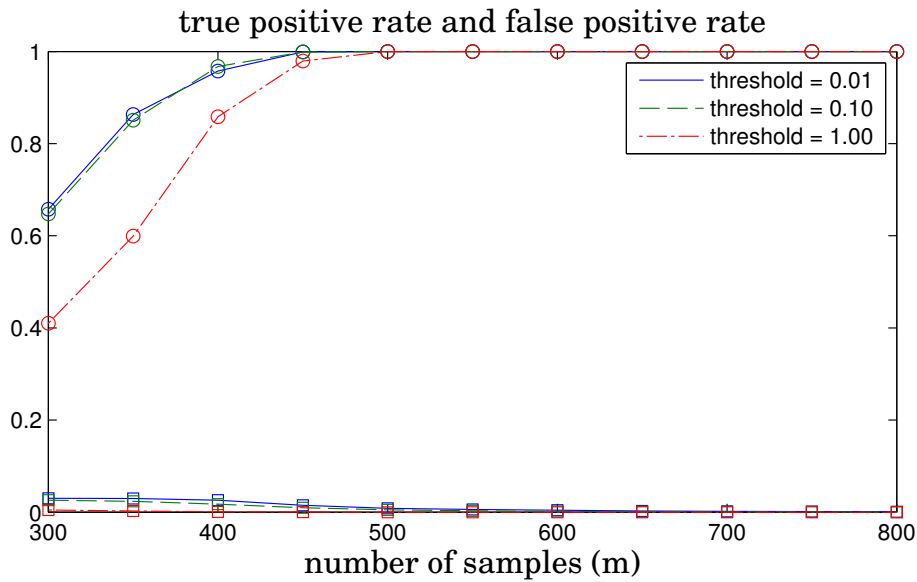}
        \label{fig:drfp}
	\caption{Support set estimation using LASSO: for $n = 6000$, $\sigma = 0.1$, $\min \vert x_i \vert \geq 3.34$, threshold $\xi_1 = 0.01$, $0.10$ and $1.00$. Circles depict \emph{true positive rate}, and squares depict \emph{false positive rate}.}
	\label{fig:drfp}
	\vspace{-10pt}
\end{figure}
\subsection{Reconstruction Error}

As was discussed in Section \ref{subsec:voting}, LASSO can be used together with a voting strategy and least squares estimation to reduce error variance. Figure~\ref{fig:voting} shows the comparison of performance of a) averaged LASSO estimates, b) debiasing and averaging with voting strategy, and c) debiasing and averaging without voting. The figure is obtained by using fixed $\x$ (i.e., a single window) and taking multiple measurements (each being an $m$-dimensional vector) corrupted by i.i.d. Gaussian noise. It can be seen that the error does not decrease to zero for averaged estimate, which is due to LASSO being a biased estimator, cf. Section \ref{sec:recursivecs}, whereas for the proposed schemes it does.
\begin{figure}
\centering
\includegraphics[height = 0.35\textheight]{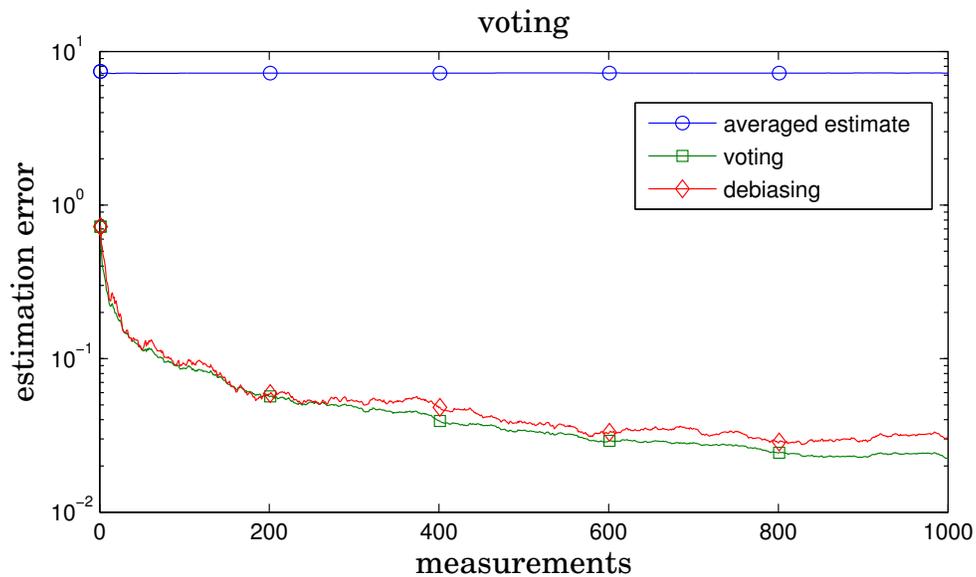}
\caption{Error plots for a) averaged estimates, b) voting strategy, and c) debiasing  without voting.}
\label{fig:voting}
\vspace{-10pt}
\end{figure}

Figure~\ref{fig:votingStreaming} shows the behavior of normalized error variance
\begin{align*}
\lim_{T \rightarrow \infty }\frac{\sum_{i=1}^T (\bar{x}_i-x_i)^2}{\sum_{i=1}^T (x_i)^2}
\end{align*}
as the window length, $n$, increases. The signals are generated to be $5\%$ sparse, $m$ is chosen to be 5 times the expected window sparsity, and the measurement noise is $\w^{(i)} \sim \mathcal{N}(0,\sigma^2 \vec{I})$ where $\sigma = 0.1$. The non-zero amplitudes of the signal are drawn from uniform distribution 
$\mathcal{U}\left( \left[-2, -1\right] \cup \left[1, 2\right] \right)$
The figure shows that the normalized error variance decreases as the window length increases, which is in full agreement  with our theoretical analysis.
\begin{figure}
\centering
\includegraphics[height = 0.35\textheight]{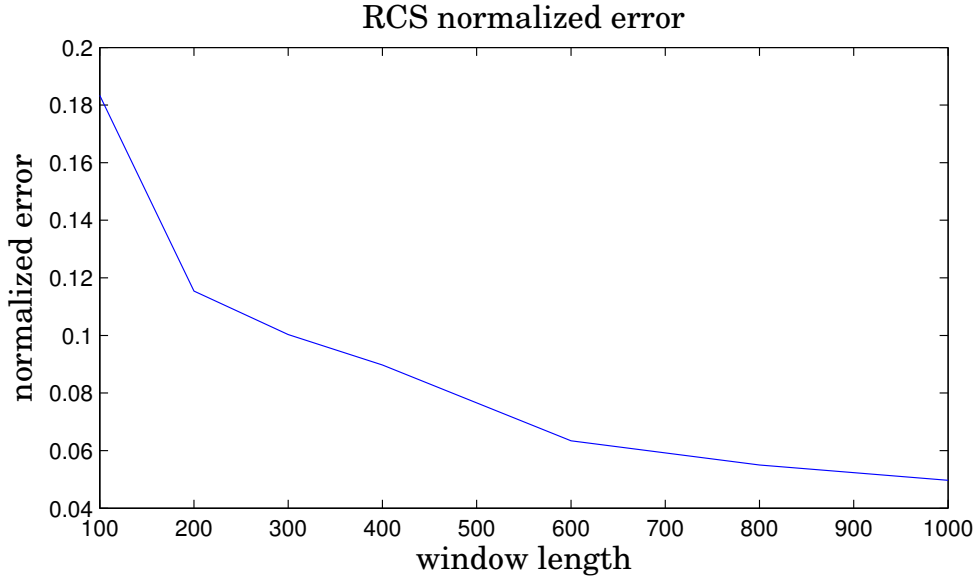}
\caption{Normalized error variance vs. window length for RCS on streaming data.}
\label{fig:votingStreaming}
\vspace{-10pt}
\end{figure}
\section{Conclusions and Future Work}
\label{sec:conclusion}

We have proposed an efficient online method for compressively sampling data streams. The method uses a sliding window for data processing and entails recursive sampling and iterative recovery. By exploiting redundancy we  achieve higher estimation accuracy as well as reduced run-time, which makes the algorithm suitable for an online implementation. Extensive experiments showcase the merits of our approach compared to traditional CS: a) at least 10x speed-up in run-time, and b) 2-3 orders of magnitude lower reconstruction error. 

In ongoing work, we study accelerating the decoding procedure by deriving a fast LASSO solver directly applicable to RCS. We also seek to apply the derived scheme in practical applications such as burst detection in networks and channel estimation in wireless communications.

\section*{Acknowledgement}
This work was supported in part by Qualcomm, San Diego, and ERC Advanced Investigators Grant, SPARSAM, no. 247006. 
\bibliographystyle{IEEEtran}
\bibliography{IEEEabrv,references}

\end{document}